\documentclass[runningheads]{llncs}

\usepackage{eccv}

\usepackage{eccvabbrv}

\usepackage{graphicx}
\usepackage{booktabs}
\usepackage{microtype}
\usepackage{multirow}
\usepackage{enumitem}
\usepackage[ruled,vlined]{algorithm2e}
\usepackage[table]{xcolor}
\newtheorem{assumption}{Assumption}

\usepackage[accsupp]{axessibility}  

\usepackage[pagebackref,breaklinks,colorlinks,citecolor=eccvblue]{hyperref}

\usepackage{orcidlink}

\begin{document}

\title{SteerFlow: Steering Rectified Flows for Faithful Inversion-Based Image Editing} 

\titlerunning{SteerFlow}




\author{Thinh Dao\inst{1,2}\thanks{Work done during Thinh Dao's visiting period at HKUST.} \and
Zhen Wang\inst{1} \and
Kien T.Pham\inst{1} \and
Long Chen\inst{1}\thanks{Corresponding author.}}

\authorrunning{Dao et al.}

\institute{The Hong Kong University of Science and Technology, Clear Water Bay, Hong Kong \and
VinUniversity, Ha Noi, Viet Nam\\
\email{\{zhenwang,longchen\}@ust.hk, \{dtdao,tkpham\}@connect.ust.hk}}

\maketitle

\newcommand{\zhen}[1]{{\color{orange}{$^\textbf{\emph{Zhen:}}$[#1]}}}
\newcommand{\edited}[1]{\textcolor{blue}{#1}}

\begin{abstract}
  Recent advances in flow-based generative models have enabled training-free, text-guided image editing by inverting an image into its latent noise and regenerating it under a new target conditional guidance. However, existing methods struggle to preserve source fidelity: higher-order solvers incur additional model inferences, truncated inversion constrains editability, and feature injection methods lack architectural transferability. To address these limitations, we propose \textbf{SteerFlow}, a model-agnostic editing framework with strong theoretical guarantees on source fidelity. In the forward process, we introduce an \textit{Amortized Fixed-Point Solver} that implicitly straightens the forward trajectory by enforcing velocity consistency across consecutive timesteps, yielding a high-fidelity inverted latent. In the backward process, we introduce \textit{Trajectory Interpolation}, which adaptively blends target-editing and source-reconstruction velocities to keep the editing trajectory anchored to the source. To further improve background preservation, we introduce an \textit{Adaptive Masking} mechanism that spatially constrains the editing signal with concept-guided segmentation and source-target velocity differences. Extensive experiments on FLUX.1-dev and Stable Diffusion 3.5 Medium demonstrate that SteerFlow consistently achieves better editing quality than existing methods. Finally, we show that SteerFlow extends naturally to a complex multi-turn editing paradigm without accumulating drift.\footnote{https://github.com/thinh-dao/SteerFlow.git}
\end{abstract}

\section{Introduction}


Recent Rectified Flow (RF) models have achieved state-of-the-art image generation by learning a linear transport trajectory between a noise distribution (\eg, Gaussian distribution) and a data distribution~\cite{liu2022flow,lipman2022flow}. Beyond generation, pretrained text-to-image (T2I) RF models~\cite{flux1dev,sd35medium} have been adapted for \textit{training-free}, \textit{text-guided} image editing — editing a \textit{source} image to faithfully follow a \textit{target} prompt while preserving its original structure. Successful editing thus requires a delicate balance between target alignment and source consistency.
\begin{figure}[t]
    \centering
    \includegraphics[width=1.0\textwidth]{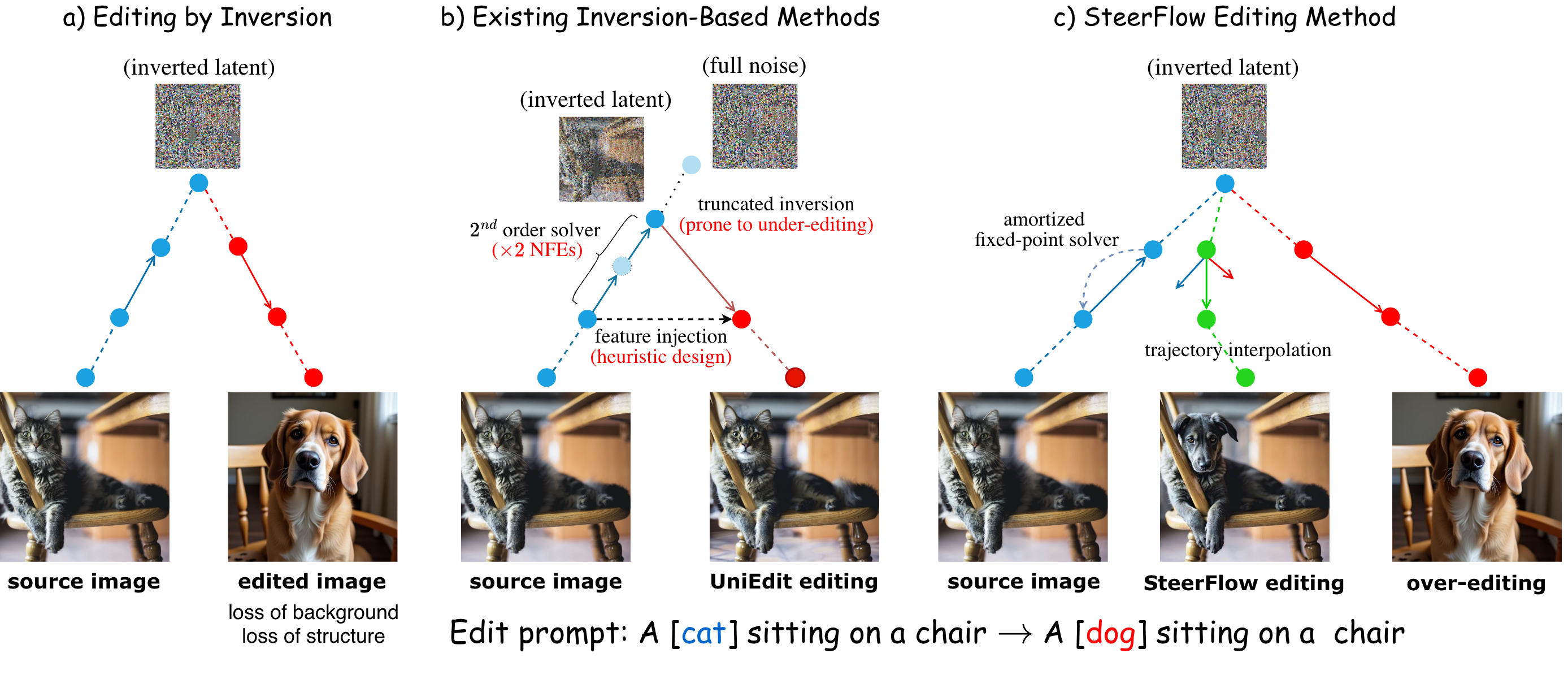}
    \caption{Comparison of SteerFlow and existing inversion-based methods. (a) Unconstrained generation causes over-editing with loss of source structure. (b) Prior approaches, represented by UniEdit~\cite{jiao2026unieditflow}, either rely on computationally expensive second-order solvers, or heuristic feature injection/truncated inversion that often suffers from under-editing (as shown by the unchanged source subject) and limited architectural transferability. (c) In contrast, SteerFlow presents a principled, architecture-agnostic editing framework that achieves high source-fidelity and faithful target alignment.}
    \label{fig:method_comparison}
\end{figure}

Leveraging the learned transport between data and noise distributions, this task is commonly addressed via a \emph{two-stage} process (\cref{fig:method_comparison} (a)): \textbf{Forward Inversion} maps the source image to an \textit{inverted latent} in the noise distribution, and \textbf{Backward Generation} synthesizes the edited image from this latent under the \textit{target} prompt~\cite{rout2024rfinversion,wang2024taming,deng2024fireflowfastinversionrectified,jiao2026unieditflow}. The quality of such editing critically depends on inversion accuracy: if the inverted latent decodes back perfectly to the source image, we can reasonably expect the target-conditioned generation to produce controlled edits with strong structural preservation. However, this editing paradigm faces two main challenges that degrade source fidelity in practice. (i) \textbf{Inversion Error:} Euler discretization of the ODE introduces truncation errors at each step that result in compounded velocity mismatch between forward and backward trajectories. (ii) \textbf{Trajectory Divergence:} Even with a perfect inverted latent, the inherent disparity between source and target velocities can still cause the editing outcome to drift significantly from the original image. As shown in \cref{fig:method_comparison} (a), unconstrained generation causes the generated dog image to lose source structural information (the cat's pose, the chair, and the table).

To address (i), recent works introduce higher-order solvers~\cite{ma2025adamsbashforthmoultonsolver,deng2024fireflowfastinversionrectified,wang2024taming} designed to reduce the order of local truncation error. While they improve inversion accuracy, they come at the cost of additional number of function evaluations (NFEs). To address (ii), a straightforward solution is \textbf{truncated inversion}, which halts the forward process at an intermediate timestep where source structure is still partially preserved~\cite{jiao2026unieditflow,rout2024rfinversion}. Another related approach is \textbf{feature injection}, which extracts internal model representations from the source forward pass (\eg, attention maps~\cite{prompt2prompt,cao2023masactrl} and value features~\cite{wang2024taming,deng2024fireflowfastinversionrectified} of select transformer blocks) and injects them into the early steps of the backward generation process. However, since these strategies are heuristic by design, they struggle to achieve consistent performance. Truncated inversion is sensitive to the choice of truncation timestep, where stopping too early causes over-editing and stopping too late causes under-editing. Feature injection requires manual tuning of which blocks to inject and for how many steps, limiting their transferability across different architectures.

To address the limitations of these heuristic designs, we rigorously analyze the \textit{inversion error} and \textit{editing error} of flow-based image editing. Our findings motivate the design of \textbf{SteerFlow}, a \textit{model-agnostic} editing framework with strong theoretical guarantees that achieves a faithful and controllable trade-off between source fidelity and target alignment. Specifically, as shown in \cref{fig:method_comparison}~(c), in the forward process, we bound the inversion error of Euler's method and introduce an efficient \textbf{Amortized Fixed-Point Solver} that implicitly straightens the forward trajectory by enforcing velocity consistency across consecutive timesteps, thereby reducing the forward-backward velocity mismatch and minimizing inversion error. In the backward process, we justify the need for source anchoring and construct an editing velocity via \textbf{Trajectory Interpolation}, which adaptively blends target-editing and source-reconstruction velocities. This design addresses trajectory divergence, ensuring that the editing outcome preserves source structure while still satisfying the target prompt. Trajectory Interpolation naturally induces an \textbf{Adaptive Masking} mechanism that acts as a \textit{spatial gate} on editing velocity, thereby improving editing localization and background preservation. This mechanism first obtains a base mask from SAM3~\cite{carion2025sam3segmentconcepts} through the source editing tokens, and then adaptively refines it with source-target velocity differentials during the backward process to produce precise, content-aware edit boundaries. Extensive experiments on PIE-Bench~\cite{ju2023direct} using FLUX.1-dev~\cite{flux1dev} and Stable Diffusion 3.5 Medium~\cite{sd35medium} demonstrate that SteerFlow consistently achieves a better trade-off between source preservation and target alignment than existing methods. Building on this foundation, we further extend SteerFlow to complex multi-turn editing, enabling iterative and compositional image modifications without accumulating significant drift from the source image. 

In summary, our main contributions are as follows:
\begin{enumerate}
\item We provide a formal analysis of Inversion Error and Trajectory Divergence, establishing error bounds that motivate our framework design and yield theoretical guarantees on source fidelity.

\item We propose SteerFlow, a principled editing framework that consists of an \textbf{Amortized Fixed-Point Solver} in the forward process and \textbf{Trajectory Interpolation} in the backward process that altogether minimizes trajectory divergence. SteerFlow further introduces an \textbf{Adaptive Masking} mechanism that constrains editing leakage and improves background preservation.

\item Extensive experiments on PIE-Bench~\cite{ju2023direct} show that SteerFlow outperforms existing methods in source preservation and target alignment. Furthermore, we show that SteerFlow naturally extends to multi-turn editing without accumulating drift across editing turns.

\end{enumerate}


\section{Related Works}
\label{sec:related}
\subsection{Inversion-Based Image Editing} 
Inversion-based image editing has been explored extensively for diffusion models via DDIM inversion~\cite{song2020denoising,couairon2023diffedit,mokady2022nulltext,masactrl,miyake2024negative}. Recent works adapt this to Rectified Flow (RF) models, exploiting their straighter transport paths to improve inversion accuracy. For example, \cite{rout2024rfinversion} proposed \textit{controller guidance} that simultaneously minimizes the deviation of the inverted latent from the noise prior and the editing output from the source image. Another line of work focuses on improving reconstruction fidelity through higher-order ODE solvers~\cite{deng2024fireflowfastinversionrectified,wang2024taming,ma2025adamsbashforthmoultonsolver}, or by refining the inverted latent with additional correction steps~\cite{jiao2026unieditflow}. While these approaches can improve image reconstruction, they are not guaranteed to ensure high source consistency for image editing since they do not resolve the \textbf{trajectory divergence} issue between source and target velocity fields. 

To enforce source preservation, previous methods often employ truncated inversion, which halts the inversion process at an earlier timestep and starts the backward process from there~\cite{jiao2026unieditflow,rout2024rfinversion}. Another line of work investigates feature injection, which transfers attention features (\eg, Values or Attention Maps) from the forward velocities to the backward velocities, typically in the first few steps~\cite{prompt2prompt,cao2023masactrl,brack2024leditspp,avrahami2024stable,deng2024fireflowfastinversionrectified,wang2024taming,zhu2025kv}. However, since both approaches rely on architecture-specific heuristics (\eg, manual selection of injection layers or truncated steps), they lack generalization across different architectures.


\subsection{Inversion-Free Image Editing}
To avoid reliance on accurate inversion, recent works have explored \emph{inversion-free} editing for RF models~\cite{kulikov2025flowedit,kim2025flowalign,yoon2025splitflow,wang2025flowcycle,jiang2025flowdcflowbaseddecouplingdecaycomplex}. FlowEdit~\cite{kulikov2025flowedit} pioneered this direction by replacing the inverted latent with a random pairing of the source image and Gaussian noise sampled at every timestep. Although FlowEdit could maintain strong structural consistency in certain cases, the method is inherently stochastic, so the final output can vary significantly across different random seeds. Building on FlowEdit, FlowAlign~\cite{kim2025flowalign} addresses the unstable editing trajectories inherent to inversion-free methods by introducing a terminal point regularization term that explicitly balances semantic alignment with the edit prompt and structural consistency with the source image. 

\textbf{SteerFlow} is an inversion-based editing method that departs from previous architecture-specific heuristics by framing backward generation as a trajectory-control problem. By dynamically anchoring the editing path to the source reconstruction, we guarantee strong structural consistency and target alignment.

\section{Preliminaries}
\subsection{Rectified Flow}
Rectified Flow (RF)~\cite{liu2022flow} is a generative framework that constructs a linear transport map from the data distribution $\pi_0 = p_{\text{data}}$ to the noise distribution $\pi_1 = \mathcal{N}(0, I_d)$. Given a clean image $Z_0 \sim \pi_0$ and Gaussian noise $Z_1 \sim \pi_1$, RF defines a linear interpolation path $Z_t = tZ_1 + (1-t)Z_0$ between coupled samples $(Z_0, Z_1) \sim \pi_0 \times \pi_1$. The probability flow is governed by the ODE $dZ_t = v_\theta(Z_t, t)\,dt$, where $v_\theta: \mathbb{R}^d \times [0,1] \to \mathbb{R}^d$ is a learnable velocity field trained with the Conditional Flow Matching (CFM) objective:
\begin{equation}
    \mathcal{L}_{\text{CFM}}(\theta) = \mathbb{E}_{t \sim \mathcal{U}[0,1]} 
    \left[ \| v_\theta(Z_t, t) - (Z_1 - Z_0) \|^2 \right].
    \label{eq:cfm}
\end{equation}

\noindent\textbf{Inference.} For text-to-image generation, the velocity field is additionally conditioned on a text prompt $c$, yielding $v_\theta(Z_t, t, c)$. Image generation is achieved by solving the ODE backward from $t=1$ to $t=0$. In practice, this continuous process is approximated using a numerical solver, such as the Euler method:
\begin{equation}
    Z_{t-\Delta t} = Z_t - v_\theta(Z_t, t, c)\,\Delta t.
    \label{equ:rf_generation_process}
\end{equation}

\noindent\textbf{Classifier-Free Guidance (CFG).} To enhance prompt alignment, 
CFG~\cite{ho2021classifierfree} scales the conditional velocity at inference:
\begin{equation}
\tilde{v}_\theta(Z_t, t, c) = v_\theta(Z_t, t, \varnothing) + 
w\left(v_\theta(Z_t, t, c) - v_\theta(Z_t, t, \varnothing)\right),
\end{equation}
where $w > 1$ is the guidance scale that amplifies the conditional signal.

\subsection{Inversion-based Image Editing}
In text-guided image editing, we start with a \textit{source} image $Z_0^{src}$ and an optional \textit{source} prompt $c_{src}$ that describes $Z_0^{src}$. The goal is to transform the source image into the \textit{target} image that follows the provided \textit{target} prompt $c_{tar}$. The editing process consists of two stages:

\noindent\textbf{1. Inversion (Forward Process).} The source image $Z_0^{src}$ is mapped to \textit{inverted latent} $Z_1^{src}$ by integrating the ODE from $t=0$ to $t=1$ conditioned on the source prompt $c_{src}$. The forward process does not use CFG ($w = 1$) to keep the latent on the data manifold:
\begin{align}
\label{eq:forward}
Z_1^{src} = Z_0^{src} + \int_{0}^{1} v_\theta(Z_t^{src}, t, c_{src}) \, dt && \text{(Inversion)}
\end{align}

\noindent\textbf{2. Generation (Backward Process).} Starting from the inverted noise $Z_1^{tar} = Z_1^{src}$, we solve the reverse-time ODE from $t=1$ to $t=0$ conditioned on the target prompt $c_{tar}$. The backward process typically uses CFG ($w > 1$) to enhance prompt adherence of the edited image:
\begin{align}
\label{eq:edit}
Z_0^{tar} &= Z_1^{src} - \int_{0}^{1} \tilde{v}_\theta(Z_t, t, c_{tar}, w) \, dt && \text{(Editing)}
\end{align}

\section{Methodology}

\begin{figure}[t]
    \centering
    \includegraphics[width=1.0\textwidth]{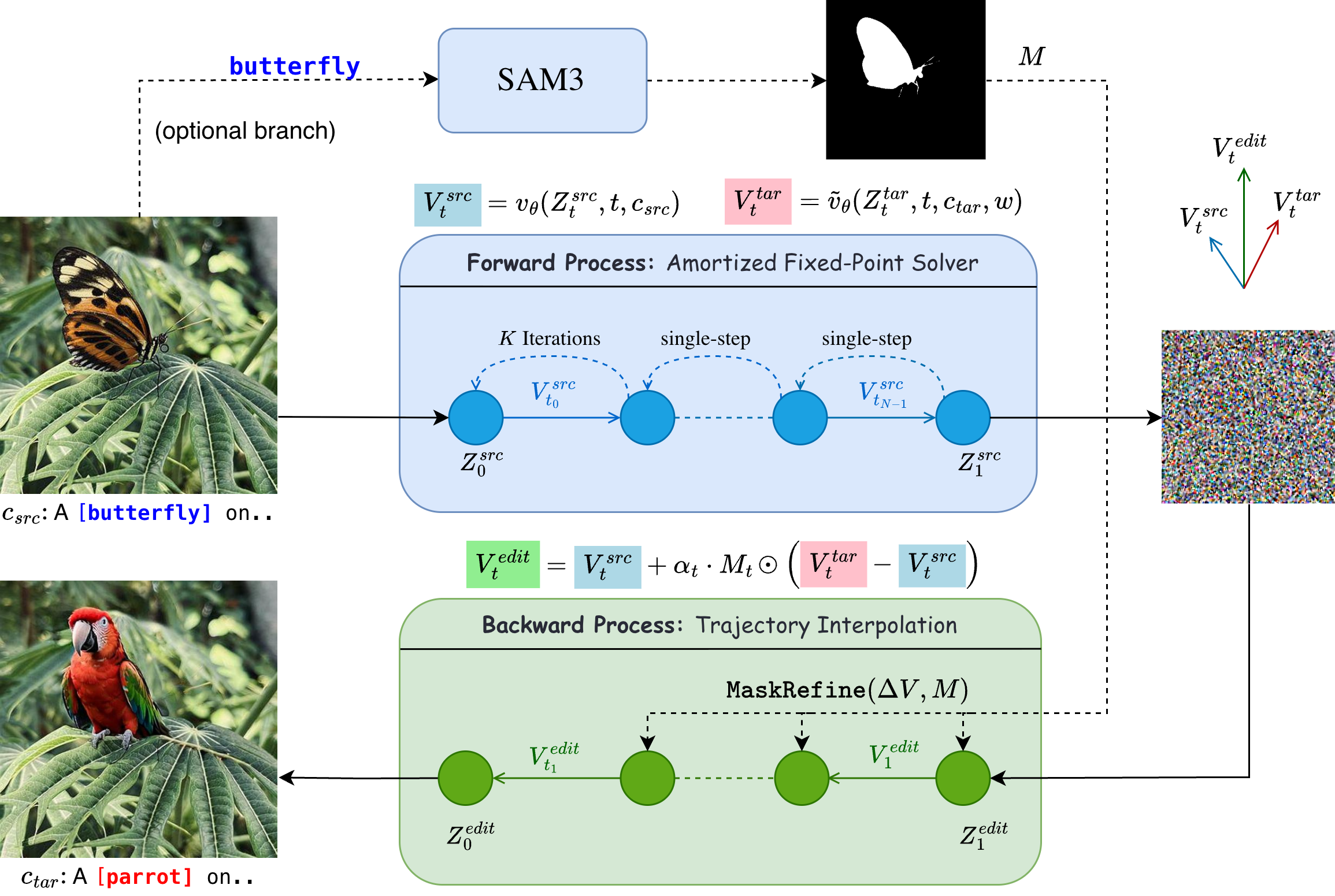}
    \caption{SteerFlow Editing Framework. Note that adaptive masking is optional.}
    \label{fig:steerflow_pipeline}
\end{figure}

\subsection{Overview}
In this section, we present the SteerFlow editing framework (see \cref{fig:steerflow_pipeline} for an overview). For both image reconstruction and editing, we derive theoretical error bounds between the output and the source image, and show that SteerFlow provably achieves tighter bounds than standard inversion-based approaches. 

\subsection{Forward Process: Amortized Fixed-Point Solver}
The primary objective of the forward process is to find an inverted latent that retains source structure so that simulating a backward ODE on this latent under the target prompt does not significantly degrade source consistency. The quality of the inverted latent is often assessed via an image reconstruction task. From \cref{eq:forward}, we solve for $Z_0^{src}$ with the inverted latent $Z_1^{src}$:
\begin{align}
\label{eq:recon}
\hat{Z}_0^{src} &= Z_1^{src} - \int_{0}^{1} v_\theta(Z_t, t, c_{src}) \, dt && \text{(Reconstruction)}
\end{align}
While this continuous integral yields exact reconstruction, it is intractable in practice. The ODE is therefore solved using Euler's method with discretized steps: $Z_{t_{i+1}} = Z_{t_i} + \Delta t \cdot v_\theta(Z_{t_i}, t_i)$, where $\{t_i\}_{i=0}^N$ represents a time schedule with $t_0 = 0$, $t_N = 1$, and $\Delta t = 1/N$. This discretization introduces an asymmetry between the forward inversion and backward reconstruction processes. Specifically, at any step $i$, the forward latent $Z_{t_{i}}$ and backward latent $\hat{Z}_{t_i}$ are computed as:
\begin{align}
\label{eq:discretization}
\text{Forward Latent: } \quad Z_{t_i} &= Z_{t_{i+1}} - \Delta t \cdot v_\theta(Z_{t_i}, t_i) \\
\text{Backward Latent: } \quad \hat{Z}_{t_i} &= \hat{Z}_{t_{i+1}} - \Delta t \cdot v_\theta(\hat{Z}_{t_{i+1}}, t_{i+1})
\end{align}
Perfect reconstruction ($\hat{Z}_{t_i} \equiv Z_{t_i}$) would require the velocity field to be constant over the interval $[t_i, t_{i+1}]$, implying a linear trajectory where $v_\theta(Z_{t_i}, t_i) = v_\theta(Z_{t_{i+1}}, t_{i+1})$. However, due to the high-dimensional curvature of pretrained flow-based models~\cite{liu2022flow}, the velocity changes dynamically along the path, creating a discretization gap $\delta_i = \| v_\theta(Z_{t_i}, t_i) - v_\theta(Z_{t_{i+1}}, t_{i+1}) \|$. This velocity mismatch induces \textit{inversion error} that is compounded as $t \rightarrow 0$. We bound the final inversion error with the following proposition: 

\begin{proposition}[Euler Inversion Error Bound]
\label{prop:inversion_error}
Let $v_\theta(z, t)$ be $L$-Lipschitz continuous in $z$ and have a bounded total time derivative $\|\frac{d v_\theta}{d t}\| \le M$, which represents the \textbf{maximum curvature} (acceleration) of the trajectory. Consider a discretization $\Delta t = 1/N$ and ignore the higher-order error $O(\Delta t^2)$. The Euler inversion error is bounded by:
\begin{equation}
    \mathcal{E}^{inv}_{0} = \| \hat{Z}_{t_0} - Z_{t_0} \| \le \frac{M \Delta t}{L} \left( (1 + L\Delta t)^N - 1 \right) \leq \frac{M \Delta t}{L} (e^L - 1)
\end{equation}
\end{proposition}
\cref{prop:inversion_error} shows that the inversion error expands at a rate of $(1 + L\Delta t)$ per step, and it depends on the maximum curvature $M$. If $M=0$ (the path is straight), $\mathcal{E}_{N}=0$. While prior higher-order methods reduce this bound to $O(\Delta t^2)$~\cite{wang2024taming,ma2025adamsbashforthmoultonsolver}, they do so at the cost of $\times$2 NFEs. 

In this work, we propose a new approach that reduces inversion error by \textbf{\textit{implicitly} straightening the forward trajectory}, which consequently reduces the forward-backward velocity mismatch. This is achieved by enforcing a consistent velocity across consecutive timesteps. Formally, at any timestep $t_i$, we solve for a $v^*_i$ that satisfies the implicit fixed-point condition: 
\begin{equation}
\label{eq:implicit_condition}
v^*_i = v_\theta(Z_{t_i} + \Delta t \cdot v^*_i, t_{i+1})
\end{equation}
To solve for $v^*_i$, we define an update map $\Phi_i: \mathbb{R}^d \to \mathbb{R}^d$ on the velocity $v_i$ as:
\begin{align}
    \label{eq:update_map}
    v_i^{k+1} &= \Phi_i(v_i^{k}) = v_\theta(Z_{t_i} + \Delta t \cdot v_i^{k},\, t_{i+1}), \\
    \text{with } v_i^0 &= v_\theta(Z_{t_i},\, t_i) \quad \text{(initial condition)}
\end{align}
We note that \cref{eq:update_map} represents a fixed-point iteration, initialized with the standard Euler velocity $v_\theta(Z_{t_i}, t_i)$. As established in \cref{prop:convergence} (App.~\ref{appendix:proofs}), when the step size $\Delta t < \frac{1}{L}$, this update map acts as a strict contraction mapping, meaning that the local velocity mismatch strictly decreases after each iteration. However, we observe that increasing the number of fixed-point iterations $K$ yields diminishing returns, as the reconstruction performance fluctuates rather than improves monotonically with $K$. We attribute this to a \textbf{moving target problem}: the fixed-point solver in \cref{eq:update_map} optimizes $v_i$ toward a target $v_{i+1}$ estimated via a zero-step Euler iteration. If $v_{i+1}$ is subsequently refined using the same procedure, the shift in the target velocity would invalidate the previous step’s optimization, leading to trajectory drift and redundant function evaluations.

\textbf{Amortized Fixed-Point (AFP) Solver.} Since the moving target problem stems from applying uniform iterative refinement across all timesteps, we address this by \textit{amortizing} the refinement cost: we apply $K$ iterations only at the first timestep to strictly minimize the initial velocity mismatch, and a \textbf{single iteration} at all subsequent timesteps, reusing the previous step's velocity as the ``warm-start'' initial guess. Single-step iteration constrains the target movement, while the warm-start ensures temporal consistency across steps. Consequently, the total overhead is only $K-1$ additional NFEs at the first timestep, compared to $N(K-1)$ for the standard fixed-point solver or $\times2$ NFEs for second-order solvers, making AFP significantly more efficient and empirically more stable. We experimentally validate this in our ablation study in App.~\ref{appendix:ablation}, where the AFP solver achieves consistently improved reconstruction across all metrics as $K$ grows, whereas the standard fixed-point solver exhibits high variance and overall worse performance across $K$. \cref{fig:recon_compare} shows an example where increasing $K$ progressively improves reconstruction quality, which is evidenced in the refinement of the girl's dress. The pseudo-code for SteerFlow Forward Inversion is in \cref{alg:steerflow_forward}, which returns the full forward trajectory for Trajectory Interpolation.

\begin{figure}[t]
    \centering
    \includegraphics[width=1.0\textwidth]{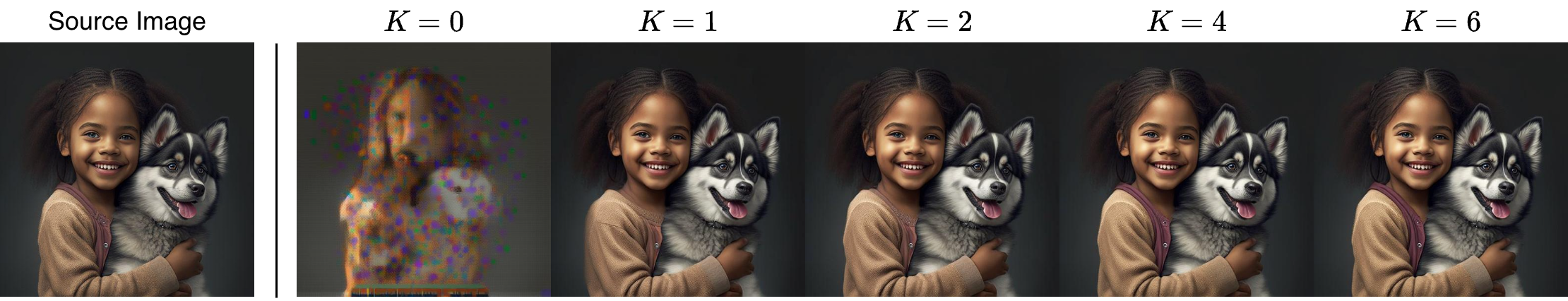}
    \caption{An illustration of the AFP Solver with increasing iterations.}
    \label{fig:recon_compare}
\end{figure}

\subsection{Backward Process: Trajectory Interpolation}
\label{sec:trajectory_interpolation}
With the AFP solver, we have constructed an inverted latent that addresses inversion error by minimizing the forward-backward velocity mismatch. However, this alone does not guarantee source fidelity in the edited output. In line with~\cite{huberman2024edit}, we find that even when starting from a \textbf{perfect latent} which generates the source image, the edited image may still deviate significantly from the source. We attribute this to \textbf{trajectory divergence}, which arises from two factors: the change in textual conditioning and the application of CFG during backward generation. We formalize this as follows:
\begin{proposition}[Backward Editing Error Bound]
\label{prop:euler_edit_error}
With $Z^{tar}_{t_N} = Z^{src}_{t_N}$, the guidance scale $w$ for the backward process, and the same discretization of step size $\Delta t$, the deviation between the source image and the target image can be decomposed into \textbf{editing deviation} and \textbf{CFG deviation}:
\begin{align}
    \label{eq:euler_edit_error}
    \mathcal{E}^{tar}_{0} &= \| Z^{tar}_{t_0} - Z^{src}_{t_0} \| \le \Delta t (||V_{\Delta}|| + (w-1)||V_{CFG}||),   \text{where } \\
  V_{\Delta} &= \sum_{i=1}^{N} v^{tar}_\theta (Z^{tar}_{t_{i}}) - v^{src}_\theta(Z^{src}_{t_{i-1}}); \; V_{CFG} = \sum_{i=1}^N v^{tar}_\theta (Z^{tar}_{t_i}) - v^{\varnothing}_\theta(Z^{tar}_{t_i})
  \label{eq:v_delta}
\end{align}
\end{proposition}
The \textbf{editing deviation} $V_{\Delta}$ captures the cumulative editing direction from the source to the target image, arising from the difference between source and target text embeddings. The \textbf{CFG deviation} $V_{\text{CFG}}$ captures the cumulative effect of classifier-free guidance, which enhances prompt adherence but also amplifies divergence from the source trajectory when $w > 1$.

Prior works~\cite{jiao2026unieditflow,kulikov2025flowedit} attempted to bound \cref{eq:euler_edit_error} using \emph{truncated inversion} that effectively sets the backward velocity to the source-reconstruction velocity $(\tilde{v}^{tar}_\theta = v^{src}_{\theta})$ for some initial timesteps. While this suppresses early error accumulation, it tends to over-constrain editability because too many truncated steps prevent the target latent from detaching from the source trajectory in later timesteps. To overcome this rigidity, we propose \textbf{Trajectory Interpolation} to replace the hard truncation with an \textit{adaptive} blend between the target-conditioned and source-reconstruction velocities. Denote $Z_t^{edit}$ as SteerFlow editing latent, $V_{t_i}^{tar}=\tilde{v}_\theta(Z_{t_i}^{edit}, t_i, c_{tar}, w)$ and $V_{t_i}^{src}=v_\theta(Z_{t_{i-1}}, t_{i-1}, c_{src})$. At time $t$, SteerFlow updates the editing latent $Z^{edit}_t$ with the following editing velocity:
\begin{align}
    \label{eq:steerflow_velocity}
    V^{edit}_t = V_t^{src} + \alpha_t\cdot (V_t^{tar} - V_t^{src}), && \alpha_t \in [0,1]
\end{align}
where $\alpha_t$ continuously interpolates between full target-conditioned editing ($\alpha_t = 1$) and pure source reconstruction ($\alpha_t = 0$). Importantly, we directly reuse the source-conditioned velocity $v_\theta(Z_t^{\text{src}}, t, c_{\text{src}})$ cached during the forward process, avoiding additional NFEs while ensuring the backward trajectory remains anchored to the ground-truth source-reconstruction path. Theoretically, our formulation of trajectory interpolation provides a tunable bound on the editing error:
\begin{proposition}[SteerFlow Editing Error Bound]
\label{prop:steerflow_error}
Let $\mathcal{B}^{tar}_0$ be the upper bound of the error $\mathcal{E}^{tar}_0$ from standard Euler integration ($\alpha=1$). For a fixed $\alpha \in (0,1)$, the upper bound of the SteerFlow editing error $\mathcal{B}^{edit}_0$ satisfies:
\begin{align}
    \label{eq:steerflow_error}
    \mathcal{E}^{edit}_0 \le \mathcal{B}^{edit}_0 < \alpha \cdot \mathcal{B}^{tar}_0
\end{align}
\end{proposition}
However, a constant $\alpha$ is not ideal since it uniformly dampens the editing signal. We therefore design an \textit{adaptive}, \textit{time-decaying} schedule $\alpha_t$ that anchors the editing latent to the source trajectory in early generation steps while progressively releasing it to the target trajectory later:
\begin{align}
    \label{eq:beta_scheduler}
    \alpha_t = \text{CosSim}(V^{src}_t, V^{tar}_t)\cdot (1-t^\gamma) && t \in [0,1]
\end{align}
where $\gamma$ controls the temporal decay rate. The coefficient $\text{CosSim}(V^{src}_t, V^{tar}_t)$ introduces directional awareness. When $V^{tar}_t$ strongly aligns with $V^{src}_t$, the editing step is structurally safe, allowing a higher weight on the target velocity. When the two vectors diverge, stronger source regularization is applied to prevent structural collapse. Since the editing velocities operate in latent space with semantic information densely encoded along the channel dimension, we compute a \textbf{spatial cosine similarity} by evaluating the cosine similarity between channel vectors at each spatial location and averaging across the spatial grid. The pseudo-code for the SteerFlow Backward Editing process is shown in App.\cref{alg:steerflow_backward}.

Crucially, the AFP Solver and Trajectory Interpolation are naturally complementary. The divergence between forward and backward velocities is offset by one timestep (see $V_\Delta$ in \cref{eq:v_delta}). By enforcing velocity consistency between consecutive timesteps, the AFP Solver reduces this divergence and ensures that the cached source velocities $V_t^{src}$ and target velocities $V_t^{tar}$ operate at the same diffusion level during interpolation.

\subsection{Adaptive Masking Mechanism}
\label{sec:adaptive_mask}
\begin{figure}[t]
    \centering
    \includegraphics[width=1\linewidth]{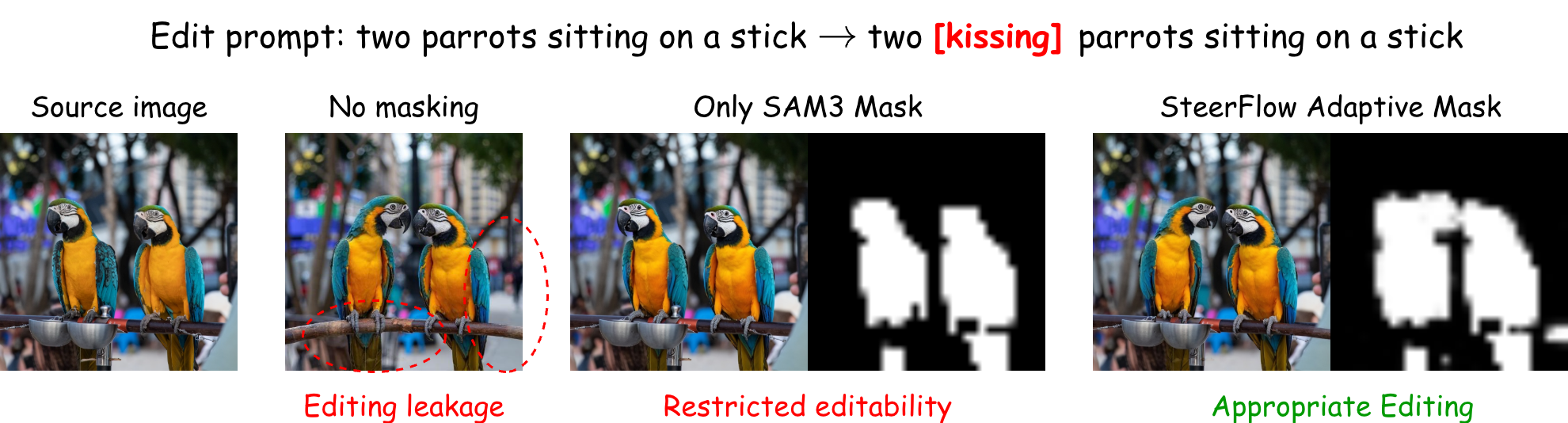}
    \caption{Comparison of masking strategies. No masking causes editing leakage, while using only the SAM3 mask restricts editability. SteerFlow Adaptive Mask dynamically expands the SAM3 mask to accommodate appropriate structural changes.}
    \label{fig:adaptive_mask}
\end{figure}

While the AFP Solver and Trajectory Interpolation altogether improve global source structural fidelity, we observe that SteerFlow still suffers from \textit{editing leakage}, where unintended regions are modified alongside the target. We attribute this primarily to the \textbf{limited concept-region localization} inherent in pretrained generative models. Because these models struggle to precisely decouple semantic concepts from their surrounding context, editing signals driven by prompt differences often bleed into the background. To address this localized failure, SteerFlow introduces an \textit{automated} masking mechanism that enforces the spatial constraints the base model lacks. We leverage SAM3~\cite{carion2025sam3segmentconcepts}, a \textit{concept-guided} segmentation model that accepts source editing tokens (\eg, \texttt{[cat]}) as input and outputs a spatially precise \textit{base} mask for the edited region. However, relying solely on the base mask is restrictive since many editing tasks require structural change beyond mask boundaries. To address this, we propose an \textit{adaptive} mask refinement procedure that expands the base mask with velocity difference $\Delta V = V^{tar} - V^{src}$. This difference captures the expected \textit{editing deviation} in image space induced by the generative model~\cite{kulikov2025flowedit}, making it suitable for mask refinement. In each backward step, the refinement synthesizes a velocity-driven mask from $\Delta V$, merges it with the SAM3 base mask, and applies morphological closing to enforce spatial contiguity. Full details are provided in \cref{alg:mask_refinement}. Importantly, the adaptive mask $M_t$ integrates seamlessly into Trajectory Interpolation (\cref{eq:steerflow_velocity}), yielding an editing velocity that is both temporally and spatially adaptive:
\begin{align}
    \label{eq:steerflow_masked_velocity}
    V^{edit}_t = V_t^{src} + \alpha_t\cdot M_t \odot (V_t^{tar} - V_t^{src}), && \alpha_t \in [0,1], M_t \in [0,1]^d
\end{align}
where $\alpha_t$ governs the global temporal strength of the edit and $M_t$ spatially gates the editing velocity so that background regions outside the mask are preserved.

\begin{figure}[t]
    \centering
    \includegraphics[width=1\textwidth]{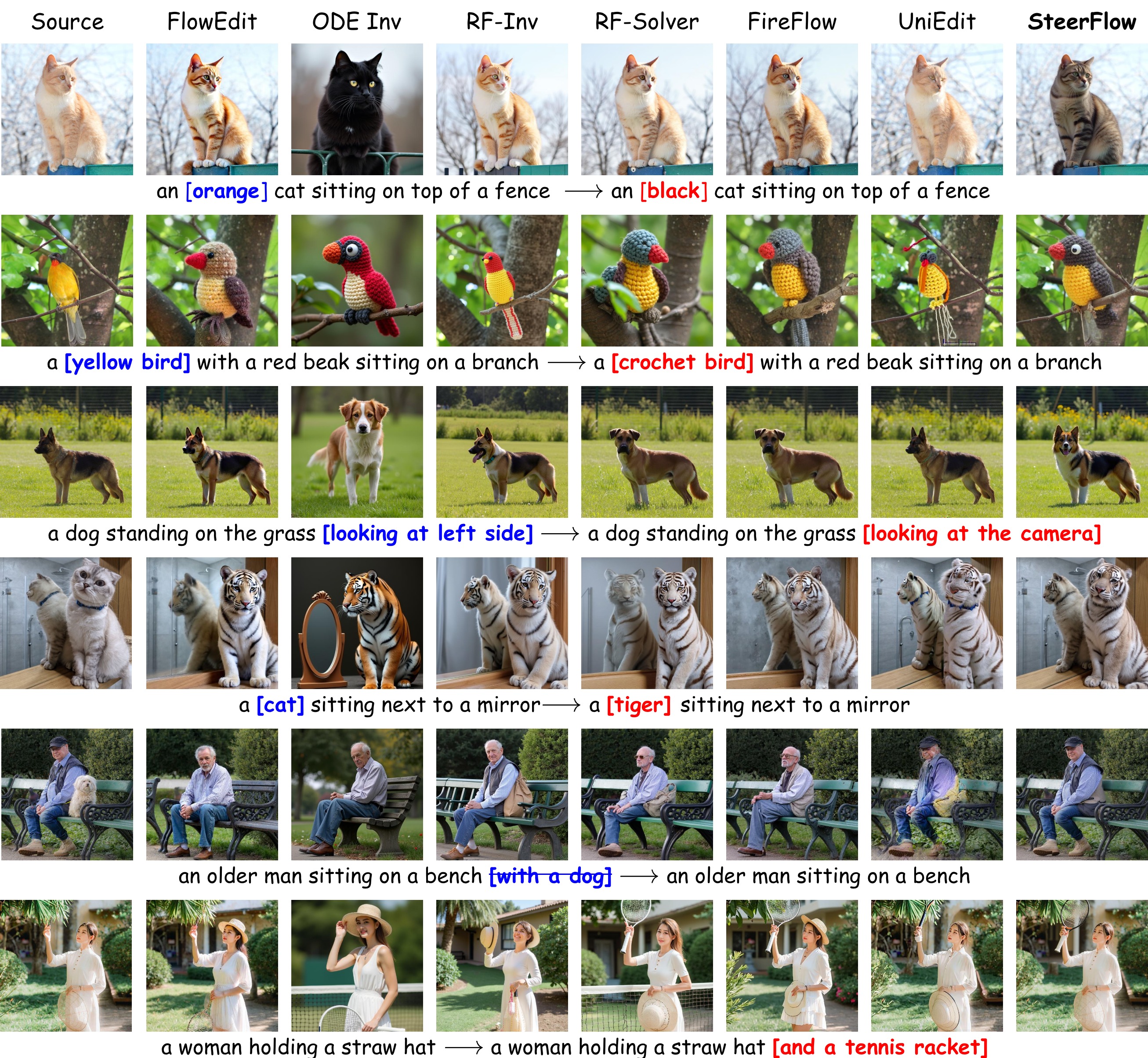}
    \caption{Qualitative comparison of editing baselines and SteerFlow for various editing tasks on FLUX.1-dev model.}
    \label{fig:visualization}
\end{figure}
\section{Experiments}
\subsection{Implementation Details}
\noindent\emph{Models and Dataset.} We evaluate our method on FLUX.1-dev~\cite{flux1dev} and 
Stable Diffusion 3.5 Medium~\cite{sd35medium}, two state-of-the-art T2I flow-based models. 
Following prior works~\cite{jiao2026unieditflow,deng2024fireflowfastinversionrectified}, we 
conduct experiments on two tasks: Image Editing (main task) and Image Reconstruction 
(secondary task to validate the AFP Solver). Both tasks are evaluated on 
PIE-Bench~\cite{ju2023direct}, a standard benchmark comprising 700 image-text pairs across 
10 editing categories. PIE-Bench also provides object masks that we use to evaluate 
background preservation.

\noindent\emph{Evaluation Metrics.} For Image Editing, we follow \cite{zhu2025kv} and assess 
performance across three dimensions. \textbf{Image Quality} is measured by 
\textbf{HPSv2}~\cite{hpsv2} and \textbf{Aesthetic Score (AS)}~\cite{aesthetic_score}, which 
quantify perceptual fidelity and visual quality. \textbf{Source Preservation} is evaluated 
via \textbf{DinoDist}~\cite{dinodist} for global structural consistency, and LPIPS, SSIM, 
and PSNR computed on unedited background regions for background preservation. 
\textbf{Target Alignment} is measured by \textbf{CLIP Score}~\cite{clip_score} for semantic 
adherence to the target prompt, and \textbf{ImageReward}~\cite{xu2023imagereward} for human 
preference alignment. For Image Reconstruction, we follow 
\cite{jiao2026unieditflow,deng2024fireflowfastinversionrectified} and measure \textbf{MSE}, 
\textbf{PSNR}~\cite{psnr}, \textbf{SSIM}~\cite{ssim}, and \textbf{LPIPS}~\cite{lpips} to 
quantify fidelity between the source and reconstructed images. To summarize overall 
performance across methods, we report the average rank (\textbf{Rank}) computed over all 
metrics.

\noindent\emph{Baselines.} We evaluate SteerFlow against two categories of baselines: (1) \textit{Inversion-Based}: \textbf{ODE Inv} (naive ODE inversion), RF-Inversion~\cite{rout2024rfinversion}, FireFlow~\cite{deng2024fireflowfastinversionrectified}, RF-Solver~\cite{wang2024taming}, and UniEdit~\cite{jiao2026unieditflow}; (2) \textit{Inversion-Free}: FlowEdit~\cite{kulikov2025flowedit} and FlowAlign~\cite{kim2025flowalign}. For Image Editing, we retain each method's originally reported configurations, as different methods exhibit varying sensitivities to hyperparameters such as guidance scale and truncation timestep. For Image Reconstruction, we standardize the number of integration steps across all methods, fixing $N=15$ for FLUX and $N=30$ for SD3.5. All baselines are evaluated using their official implementations where available; as not all methods provide open-source support for both models, some comparisons are limited to a single architecture.

\subsection{Experiment Results}
\begin{table}[t]
\centering
\caption{Quantitative comparison on PIE-Bench using FLUX.1-dev and SD3.5-Medium. \textbf{Bold} indicates best, \underline{underline} second best. DinoDist, LPIPS, SSIM, and HPSv2 are scaled $\times10^2$ for readability. For SD3.5, UniEdit uses $N=30$ steps (default is 15) to be comparable to other methods. SteerFlow uses $K=1$ optimization steps.}
\label{tab:editing_results}
\resizebox{\textwidth}{!}{
\begin{tabular}{l cc cccc c cc c cc}
\toprule
\multirow{2}{*}{\textbf{Method}} & \multirow{2}{*}{\textbf{NFE} $\downarrow$} & \multirow{2}{*}{\textbf{Rank} $\downarrow$} & \multicolumn{4}{c}{\textbf{Source Preservation}} & & \multicolumn{2}{c}{\textbf{Target Alignment}} & & \multicolumn{2}{c}{\textbf{Image Quality}} \\
\cmidrule{4-7} \cmidrule{9-10} \cmidrule{12-13}
& & & DinoDist $\downarrow$ & LPIPS $\downarrow$ & SSIM $\uparrow$ & PSNR $\uparrow$ & & CLIP $\uparrow$ & ImgRew $\uparrow$ & & HPSv2 $\uparrow$ & AS $\uparrow$ \\
\midrule
\multicolumn{13}{c}{\textbf{FLUX.1-dev}} \\
\midrule
\multicolumn{13}{l}{\textit{Inversion-Free Methods}} \\
FlowEdit~\cite{kulikov2025flowedit} & 48 & 3.62 & 2.61 & 10.50 & 84.04 & 22.37 & & 25.27 & 0.654 & & \underline{28.78} & \underline{6.74} \\
\midrule
\multicolumn{13}{l}{\textit{Inversion-Based Methods}} \\
ODE Inv & 30 & 4.50 & 10.38 & 30.64 & 64.98 & 14.52 & & \textbf{26.67} & \textbf{1.275} & & \textbf{31.25} & \textbf{6.98} \\
RF-Inversion~\cite{rout2024rfinversion} & 58 & 6.62 & 4.10 & 17.43 & 70.84 & 20.54 & & 24.62 & 0.399 & & 27.85 & 6.74 \\
FireFlow~\cite{deng2024fireflowfastinversionrectified} & 32 & 5.00 & 2.77 & 12.63 & 82.73 & 23.29 & & 25.13 & 0.647 & & 28.09 & 6.67 \\
RF-Solver~\cite{wang2024taming} & 60 & 5.00 & 3.07 & 13.66 & 81.62 & 22.82 & & 25.20 & 0.665 & & 28.14 & 6.65 \\
UniEdit~\cite{jiao2026unieditflow} & 28 & 4.62 & \textbf{0.98} & \underline{5.65} & \underline{90.92} & \textbf{29.62} & & 25.08 & 0.374 & & 26.05 & 6.28 \\
\rowcolor{gray!10} 
\textbf{SteerFlow (w/o mask)} & 31 & \underline{3.38} & 2.41 & 10.80 & 84.05 & 23.00 & & \underline{25.44} & \underline{0.682} & & 28.57 & 6.60 \\
\rowcolor{gray!10} 
\textbf{SteerFlow} & 31 & \textbf{3.25} & \underline{1.61} & \textbf{4.66} & \textbf{91.38} & \underline{27.66} & & 25.31 & 0.651 & & 28.13 & 6.59 \\
\midrule
\multicolumn{13}{c}{\textbf{SD3.5-Medium}} \\
\midrule
\multicolumn{13}{l}{\textit{Inversion-Free Methods}} \\
FlowEdit~\cite{kulikov2025flowedit} & 66& 3.75 & 3.97 & 10.11 & 82.59 & 22.10 & & 26.65 & 1.088 & & 28.11 & \underline{6.71} \\
FlowAlign~\cite{kim2025flowalign} & 66& 4.38 & 3.06 & \underline{6.44} & \underline{86.83} & 24.69 & & 25.61 & 0.722 & & 24.97 & 6.27 \\
\midrule
\multicolumn{13}{l}{\textit{Inversion-Based Methods}} \\
ODE Inv & 60& 3.50 & 12.42 & 33.21 & 63.11 & 13.13 & & \textbf{27.60} & \textbf{1.322} & & \textbf{29.29} & \textbf{6.73} \\
UniEdit~\cite{jiao2026unieditflow} & 55& 3.75 & \underline{2.12} & 9.54 & 86.47 & \underline{25.41} & & 25.97 & 0.869 & & 27.19 & 6.40 \\
\rowcolor{gray!10} 
\textbf{SteerFlow (w/o mask)} & 61& \underline{3.12} & 2.74 & 11.76 & 83.84 & 22.86 & & \underline{26.86} & \underline{1.095} & & \underline{28.78} & 6.54 \\
\rowcolor{gray!10} 
\textbf{SteerFlow} & 61& \textbf{2.50} & \textbf{1.82} & \textbf{5.09} & \textbf{90.38} & \textbf{26.92} & & 26.35 & 0.976 & & 28.03 & 6.52 \\
\bottomrule
\end{tabular}
}
\end{table}
\noindent\emph{Image Editing.} In \cref{tab:editing_results}, we report results for both SteerFlow and SteerFlow (w/o mask) to isolate the contribution of adaptive masking and to promote fair comparison with prior methods that do not incorporate masking. Even without adaptive masking, SteerFlow still achieves the best overall rankings across both models, balancing source preservation and target alignment better than all evaluated methods. Naive ODE inversion achieves the highest target alignment and image quality scores but the worst source preservation, which is expected since it applies no source preservation constraints. UniEdit achieves competitive source preservation but at the cost of constrained editability, owing to its reliance on truncated inversion with a high truncation ratio. On FLUX.1-dev, second-order methods such as RF-Solver and FireFlow offer better editability but exhibit weaker source consistency. The ablation between SteerFlow and SteerFlow (w/o mask) demonstrates that adaptive masking substantially improves source preservation with only a minor reduction in target alignment. These quantitative results are strongly reflected through qualitative visualizations in \cref{fig:visualization}.

\begin{table}[t]
\centering
\caption{Quantitative comparison on image reconstruction. \textbf{Bold} indicates the best performance, and \underline{underline} indicates the second best. LPIPS, SSIM, and MSE are scaled $\times10^2$ for readability. SteerFlow uses $K=8$ optimization steps.}
\label{tab:reconstruction_combined}
\resizebox{\textwidth}{!}{
\begin{tabular}{l l c cccc c cccc}
\toprule
\multirow{2}{*}{\textbf{Model}} & \multirow{2}{*}{\textbf{Method}} & \multirow{2}{*}{\textbf{Rank} $\downarrow$} & \multicolumn{4}{c}{\textbf{Conditional Reconstruction}} & & \multicolumn{4}{c}{\textbf{Unconditional Reconstruction}} \\
\cmidrule{4-7} \cmidrule{9-12}
& & & PSNR $\uparrow$ & LPIPS $\downarrow$ & SSIM $\uparrow$ & MSE $\downarrow$ & & PSNR $\uparrow$ & LPIPS $\downarrow$ & SSIM $\uparrow$ & MSE $\downarrow$ \\
\midrule
\multirow{5}{*}{\shortstack[c]{FLUX.1-dev \\ ($N=15$)}} 
& ODE Inv & 5.00 & 15.32 & 41.11 & 55.97 & 3.52 & & 15.65 & 48.04 & 48.11 & 3.08 \\
& RF-Solver~\cite{wang2024taming} & 3.38 & 19.32 & 20.80 & 73.34 & \underline{2.13} & & 19.39 & 21.47 & 72.01 & \underline{1.80} \\
& FireFlow~\cite{deng2024fireflowfastinversionrectified} & \textbf{1.62} & \textbf{20.45} & \underline{17.93} & 75.83 & \textbf{1.67} & & \textbf{20.26} & \underline{19.65} & \underline{74.12} & \textbf{1.58} \\
& UniEdit~\cite{jiao2026unieditflow} & 3.25 & 19.51 & 18.66 & \underline{76.12} & 3.27 & & 18.98 & 21.21 & 73.12 & 2.06 \\
\rowcolor{gray!10} 
& \textbf{SteerFlow} & \underline{1.75} & \underline{20.26} & \textbf{15.99} & \textbf{78.83} & 2.60 & & \underline{19.63} & \textbf{17.70} & \textbf{76.84} & 1.91 \\
\midrule
\multirow{5}{*}{\shortstack[c]{SD3.5-Medium \\ ($N=30$)}} 
& ODE Inv & 5.00 & 15.80 & 51.38 & 47.66 & 3.25 & & 13.60 & 59.48 & 41.08 & 4.97 \\
& RF-Solver~\cite{wang2024taming} & 4.00 & 18.06 & 41.99 & 53.88 & 2.04 & & 15.51 & 48.68 & 46.96 & 3.27 \\
& FireFlow~\cite{deng2024fireflowfastinversionrectified} & 3.00 & 23.48 & 16.84 & 77.03 & 0.86 & & 17.53 & 30.16 & 64.22 & 2.27 \\
& UniEdit~\cite{jiao2026unieditflow} & \textbf{1.25} & \textbf{24.28} & \textbf{10.09} & \textbf{82.35} & \textbf{0.56} & & \textbf{20.29} & \underline{16.69} & \textbf{73.81} & \textbf{1.27} \\
\rowcolor{gray!10} 
& \textbf{SteerFlow} & \underline{1.75} & \underline{24.25} & \underline{10.18} & \underline{82.27} & \underline{0.58} & & \underline{20.24} & \textbf{16.67} & \textbf{73.81} & \underline{1.28} \\
\bottomrule
\end{tabular}
}
\end{table}

\noindent\emph{Image Reconstruction.} As shown in \cref{tab:reconstruction_combined}, method rankings vary considerably across architectures. For instance, FireFlow ranks first on FLUX but drops to third on SD3.5, while UniEdit ranks first on SD3.5 but fourth on FLUX. In contrast, SteerFlow consistently ranks second across both models, demonstrating strong cross-architecture transferability. Notably, on FLUX, SteerFlow achieves the best perceptual quality, ranking first in LPIPS and SSIM for both conditional and unconditional reconstruction. On SD3.5, SteerFlow closely matches the top-performing UniEdit across all metrics, while outperforming all other baselines by a significant margin. Qualitative results are shown in the appendix.

\section{Extension to Multi-Turn Image Editing}
Multi-turn image editing involves applying a sequence of edits such that each subsequent outcome remains consistent with prior turns. Existing inversion-based methods struggle in this setting, as accumulated differences in textual embeddings across turns cause significant drift from the source image. In contrast, SteerFlow naturally extends to this paradigm by updating the source-anchoring trajectory with the editing trajectory from the previous turn. This requires only $N$ additional NFEs per subsequent instruction while guaranteeing that each new output does not deviate significantly from the previous one. We find that progressively decreasing $\gamma$ in the edit coefficient (\cref{eq:beta_scheduler}) after each editing turn effectively constrains accumulated drift. In turns where new objects are edited (last row of \cref{fig:multi_turn}), the SAM3 base mask should be updated with the corresponding object concept. Some qualitative results are shown in \cref{fig:multi_turn}.
\begin{figure}[t]
    \centering
    \includegraphics[width=1\textwidth]{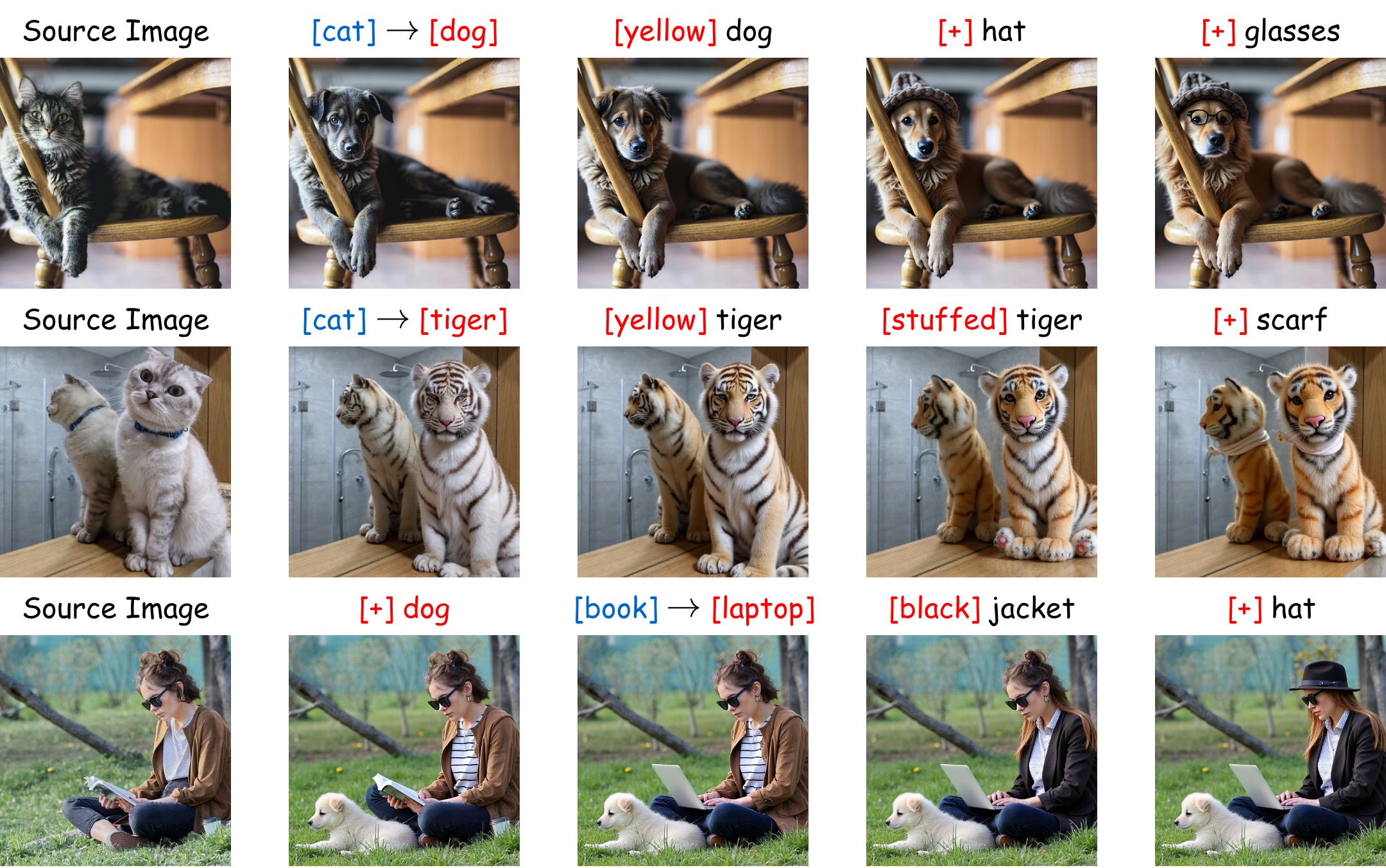}
    \caption{Multi-turn image editing with SteerFlow. Each turn applies one of three edit operations: object replacement, attribute modification, object addition.}
    \label{fig:multi_turn}
\end{figure}
\section{Conclusion}

This paper presented SteerFlow, a principled inversion-based image editing framework with theoretical guarantees on source fidelity.  Grounded in our analysis of reconstruction and 
editing error bounds, SteerFlow introduces an Amortized Fixed-Point Solver that implicitly straightens the forward trajectory to minimize inversion error, and Trajectory Interpolation that adaptively anchors editing trajectories to the source to minimize trajectory divergence. SteerFlow further improves editing localization with Adaptive Masking, which iteratively refines a base SAM3 mask using source-target velocity differences. Extensive experiments on FLUX.1-dev and SD3.5-Medium demonstrate that SteerFlow consistently achieves the best trade-off between source preservation and target alignment, outperforming state-of-the-art inversion-based and inversion-free methods. 




%
%
\bibliographystyle{splncs04}
\bibliography{main}
\newpage
\appendix
\title{Appendices for ``SteerFlow: Steering Rectified Flows for Faithful Inversion-Based Image Editing''} 
\maketitle
The appendix is organized as follows:
\begin{itemize}
    \item \cref{appendix:addition_qualitative} presents additional qualitative results for image editing, multi-turn image editing, and image reconstruction.
    \item \cref{appendix:steerflow_algorithm} presents the detailed SteerFlow editing algorithm, its hyperparameter settings, and its extension to multi-turn image editing.
    \item \cref{appendix:proofs} presents the theoretical proofs.
    \item \cref{appendix:ablation} presents the ablation study.
    \item \cref{appendix:implementation_details} presents the implementation details and hyperparameter settings for the baselines and SteerFlow.
    \item \cref{appendix:limitations} discusses the limitations of SteerFlow.
    \item \cref{appendix:perfect_latent_study} presents a study on image editing with a perfect latent.
\end{itemize}

\newpage
\section{Additional Qualitative Results}
\label{appendix:addition_qualitative}
\begin{figure}[!ht]
    \centering
    \includegraphics[width=\linewidth]{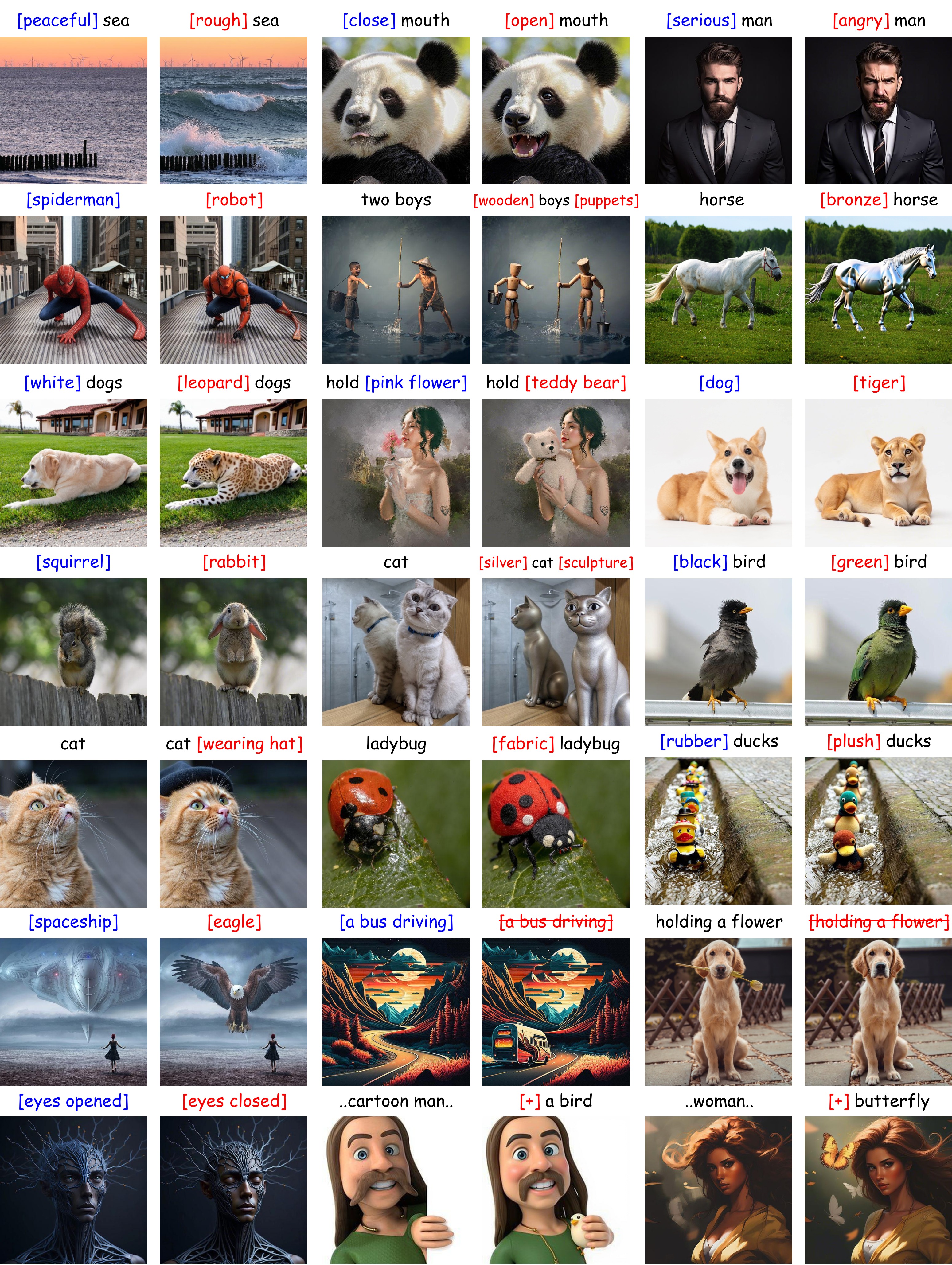}
    \caption{Additional SteerFlow results on FLUX.1-dev model.}
    \label{fig:flux_addi_vis}
\end{figure}

\begin{figure}
    \centering
    \includegraphics[width=1\linewidth]{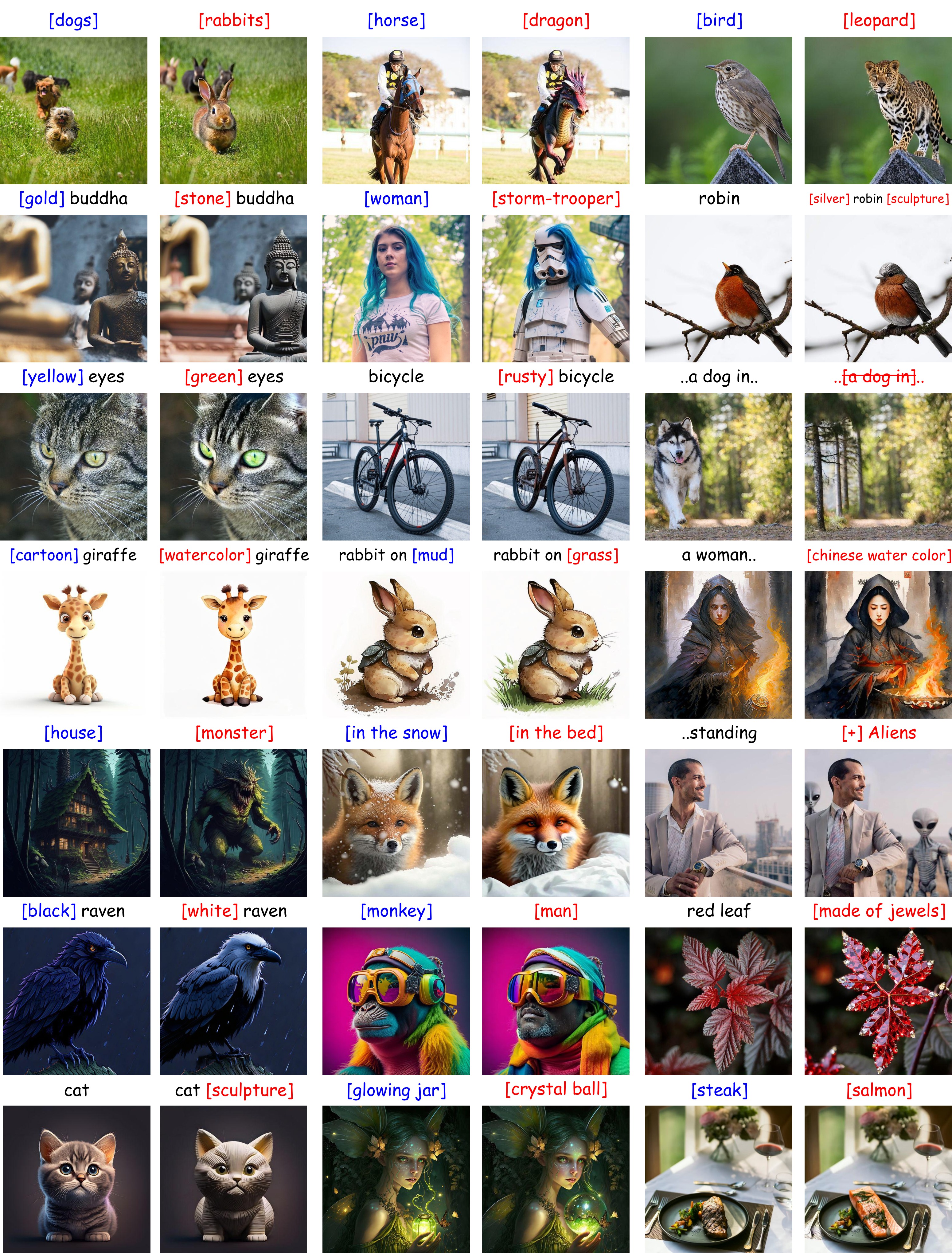}
    \caption{Additional SteerFlow results on Stable Diffusion 3.5-Medium model.}
    \label{fig:sd3_addi_vis}
\end{figure}

\begin{figure}
    \centering
    \includegraphics[width=\linewidth]{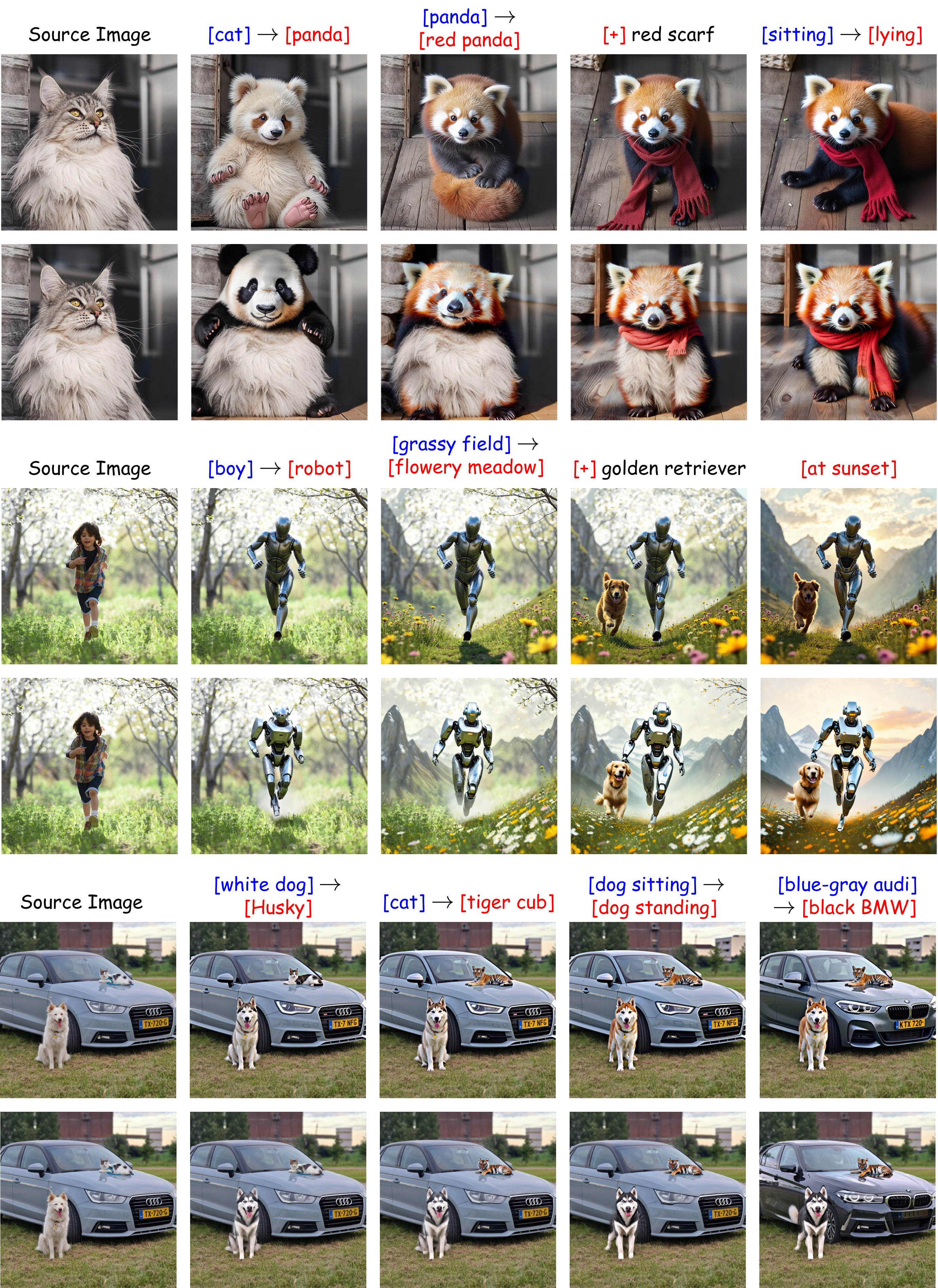}
    \caption{Additional SteerFlow results on Multi-Turn editing. First row is Flux.1-dev, second row is Stable Diffusion 3.5-Medium. For both models, SteerFlow enables consistent multi-turn editing and complex multi-object editing (in the last row).}
    \label{fig:addi_multi_turn}
\end{figure}

\begin{figure}
    \centering
    \includegraphics[width=\linewidth]{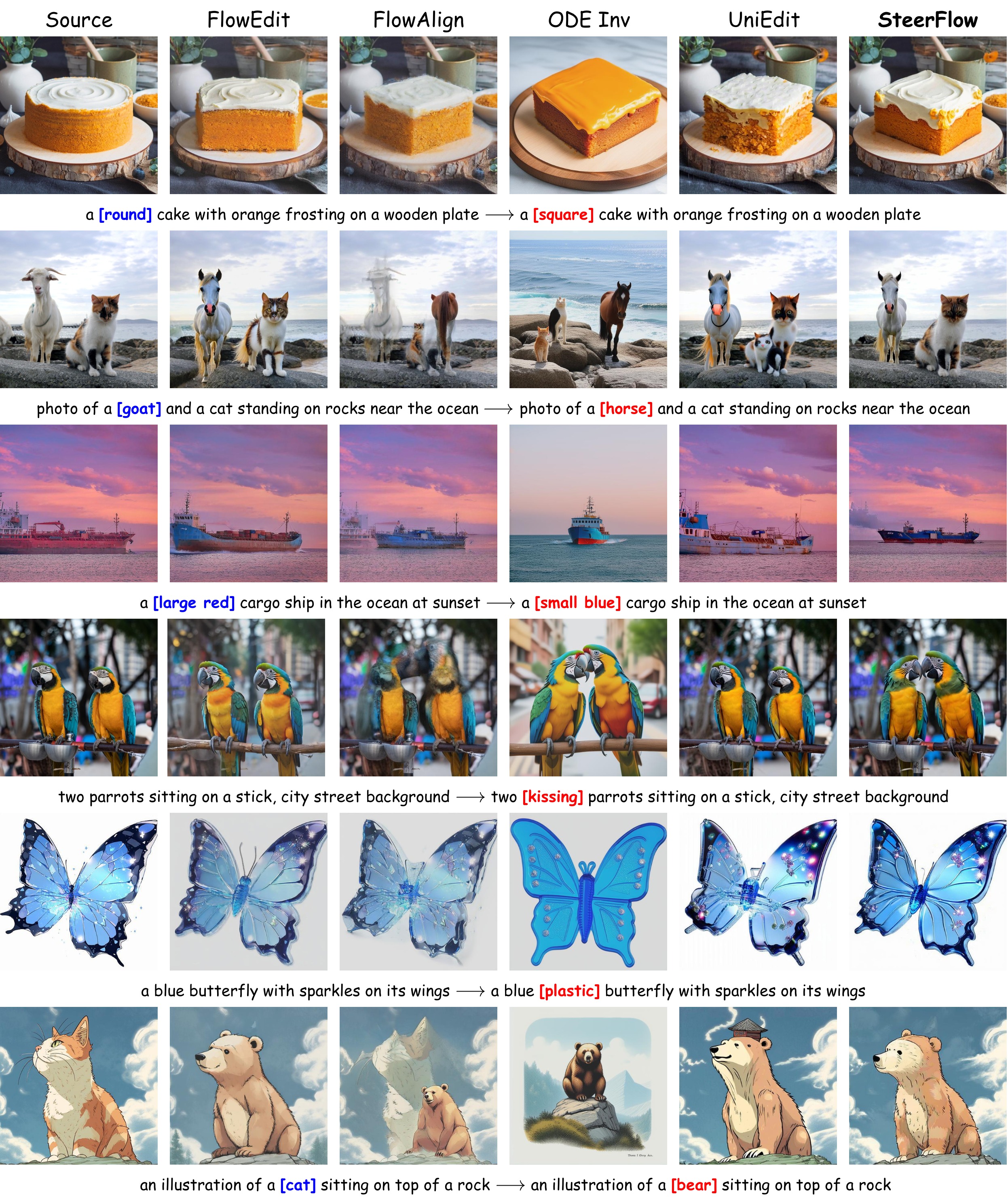}
    \caption{Qualitative comparison of baseline editing methods and SteerFlow on Stable Diffusion 3.5-Medium. SteerFlow demonstrates strong source preservation and faithful target prompt alignment. Previous editing baselines struggle with pose and structural change (in the third and fourth rows). ODE Inversion suffers from over-editing.}
    \label{fig:sd3_comparison}
\end{figure}

\begin{figure}
    \centering
    \includegraphics[width=\linewidth]{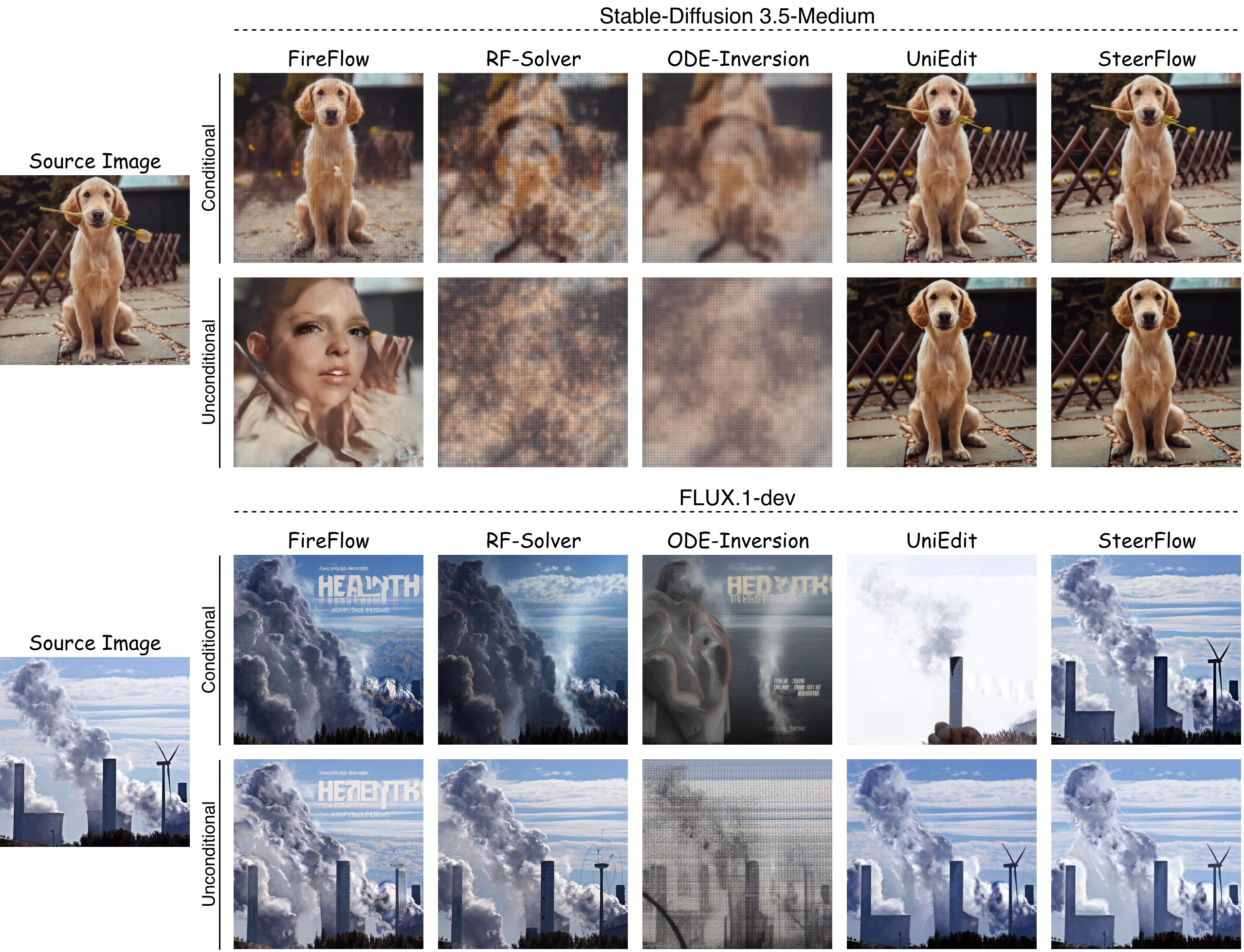}
    \caption{Qualitative comparison on inversion \& reconstruction performance of baselines and SteerFlow.}
    \label{fig:recon_performance}
\end{figure}

\newpage

\section{SteerFlow Editing Algorithm}
\label{appendix:steerflow_algorithm}
\subsection{Overall Pipeline}
The overall pipeline of SteerFlow works as follows. First, we obtain the \textbf{base mask} by feeding the source image and the editing tokens into SAM3. This process is relatively fast with almost no overhead, as we benchmark that an inference through SAM3 takes less than 1 second. Second, we perform Forward Inversion with \cref{alg:steerflow_forward} to obtain the \textbf{full forward trajectory}. With the forward trajectory and the base mask, we perform Backward Editing with \cref{alg:steerflow_backward} to obtain the \textbf{edited outcome}. At each iteration of the backward process, \cref{alg:mask_refinement} refines the base mask with a velocity-driven adaptive mask.

\begin{algorithm}[!ht]
\SetAlgoLined
\LinesNumbered
\DontPrintSemicolon
\caption{\texttt{ForwardInv} (SteerFlow Forward Inversion)}
\label{alg:steerflow_forward}
\KwIn{Initial latent $Z_{t_0}$, source prompt $c_{src}$, step size $\Delta t$, Picard iterations $K$}
\KwOut{Trajectory $\{Z_{t_i}\}_{i=0}^N$}

$V^{src} \leftarrow v_\theta(Z_{t_0}, t_0, c_{src})$ \tcp*{Initial velocity guess}

\tcp{Amortized Initialization at $t_0$}
\For{$k = 1$ \KwTo $K$}{
    $\tilde{Z}_{t_1} \leftarrow Z_{t_0} + \Delta t \cdot V^{src}$ \;
    $V^{src} \leftarrow v_\theta(\tilde{Z}_{t_1}, t_1, c_{src})$ \;
}
$Z_{t_1} \leftarrow Z_{t_0} + \Delta t \cdot V^{src}$ \;
\tcp{Forward ODE with Single-Step Fixed-Point Solver}
\For{$i = 1$ \KwTo $N-1$}{
    $\tilde{Z}_{t_{i+1}} \leftarrow Z_{t_i} + \Delta t \cdot V^{src}$ \tcp*{Predictor using cached $V$}
    $V^{src} \leftarrow v_\theta(\tilde{Z}_{t_{i+1}}, t_{i+1}, c_{src})$ \tcp*{Fixed-point correction}
    $Z_{t_{i+1}} \leftarrow Z_{t_i} + \Delta t \cdot V^{src}$
}
\Return{$\{Z_{t_i}\}_{i=0}^N$}
\end{algorithm}

\begin{algorithm}[!ht]
\SetAlgoLined
\LinesNumbered
\DontPrintSemicolon
\caption{\texttt{BackwardEdit} (SteerFlow Backward Editing)}
\label{alg:steerflow_backward}
\KwIn{Source trajectory $\{Z^{src}_{t_i}\}_{i=0}^N$, target prompt $c_{tar}$, base mask $M_{\text{base}}$, guidance scale $w$, interpolation exponent $\gamma$}
\KwOut{Edited image $Z^{tar}_{t_0}$}
$Z^{tar}_{t_N} \leftarrow Z^{src}_{t_N}$ \tcp*{Initialize from source}
\For{$i = N$ \KwTo $1$}{
    $V^{src} \leftarrow (Z^{src}_{t_i} - Z^{src}_{t_{i-1}})\,/\,\Delta t$ \tcp*{Cached source velocity}
    $V^{tar} \leftarrow \tilde{v}_\theta(Z^{tar}_{t_i}, t_i, c_{tar}, w)$ \tcp*{Target velocity with CFG}
    
    $\alpha \leftarrow \text{CosSim}(V^{src}, V^{tar})\cdot (1-{t_{i-1}}^\gamma)$ \tcp*{Edit coefficient}
    $M \leftarrow \texttt{MaskRefine}(V^{tar} - V^{src},\, M_{\text{base}})$ \tcp*{Alg.~\ref{alg:mask_refinement}}
    $V^{edit} \leftarrow V^{src} + \alpha \cdot M \odot (V^{tar} - V^{src})$ \tcp*{Eq.~\ref{eq:steerflow_masked_velocity}}

    $Z^{tar}_{t_{i-1}} \leftarrow Z^{tar}_{t_i} - \Delta t \cdot V^{edit}$ \tcp*{Euler step}
}
\Return{$Z^{tar}_{t_0}$}
\end{algorithm}

\begin{algorithm}[!ht]
\SetAlgoLined
\LinesNumbered
\DontPrintSemicolon
\caption{\texttt{MaskRefine} (Adaptive Mask Refinement)}
\label{alg:mask_refinement}
\KwIn{Velocity difference $\Delta V$, base mask $M_{base}$, quantile $q$, sigmoid temperature $\tau$, dilation kernel size $k$}
\KwOut{Refined adaptive mask $M_{adapt}$}
$M \leftarrow \|\Delta V\|_2$ \tcp*{Per-token $\ell_2$ magnitude}
$v_{\min}, v_{\max} \leftarrow \mathrm{Quantile}(M,\, 1{-}q),\; \mathrm{Quantile}(M,\, q)$ \;
$M \leftarrow (M - v_{\min})\,/\,(v_{\max} - v_{\min})$ \tcp*{(1) Quantile normalization}
$M \leftarrow \sigma\!\left(\tau \cdot (M - 0.5)\right)$ \tcp*{(2) Temperature-scaled sigmoid}
$M_{adapt} \leftarrow \max(M, M_{base})$ \tcp*{(3) Union with SAM3 base mask}
$M_{adapt} \leftarrow \mathrm{Erode}(\mathrm{Dilate}(M_{adapt},\, k),\, k)$ \tcp*{(4) Morphological closing}
\Return{$M_{adapt}$}
\end{algorithm}

Note that in the backward process, the decay factor $(1-t_{i-1}^\gamma)$ is evaluated at the next timestep $t_{i-1}$, since the decay factor evaluated at the current timestep will truncate the first timestep ($V_1^{edit} = V_1^{src}$) and limit editability.

\subsection{Hyperparameter Settings}
In the SteerFlow editing algorithm, the two key parameters that control the balance between source preservation and target alignment are (1) the CFG coefficient $w$ and (2) the decay rate of the editing coefficient $\gamma$. Higher values of $w$ and $\gamma$ lead to better editability but worse source consistency. 

\subsection{Adaptive Masking Mechanism}
As explained in \cref{sec:adaptive_mask}, relying solely on the SAM3 base mask is overly restrictive for editing tasks that involve structural changes. The SteerFlow Adaptive Masking mechanism refines the base mask with velocity-driven region changes inferred through the editing deviation $V_\Delta = V^{tar} - V^{src}$ at each timestep. As shown in \cref{alg:mask_refinement}, the refinement pipeline proceeds in four steps:
\begin{enumerate}
    \item \textbf{Quantile normalization} of $\Delta V$: controls the spatial 
    extent of the velocity-driven mask. By setting $q$ lower than 1, we allow for more expansion of the velocity-driven mask.
    
    \item \textbf{Element-wise sigmoid}: increases the contrast of the adaptive 
    mask to suppress low-response regions for precise background preservation.
    
    \item \textbf{Mask union}: merges the resulting mask with the SAM3 base 
    mask via pixel-wise $\max(\cdot)$.
    
    \item \textbf{Morphological closing}: enforces spatial contiguity by 
    filling internal holes in the combined mask.
\end{enumerate}

\subsection{SteerFlow Extension to Multi-Turn Editing}
SteerFlow extends naturally to Multi-Turn Editing through two steps:

\begin{enumerate}
    \item Reuse the previous turn's editing trajectory as the source trajectory for the current turn. This constrains the accumulated drift after each editing turn.
    \item  Update the base mask at each turn using the corresponding edit tokens.
\end{enumerate}
Since each editing turn may require a different level of editability, 
we allow different decay rates $\gamma$ per turn, resulting in a 
per-turn \textit{editing scheduler} $\{\gamma_k\}_{k=1}^{K}$ for 
$K$ editing turns. SteerFlow Multi-Turn procedure is shown in \cref{alg:steerflow_multiturn}.

\begin{algorithm}
\caption{SteerFlow Multi-Turn Editing}
\label{alg:steerflow_multiturn}
\KwIn{Source image $Z_{t_0}^{src}$, source prompt $c_{src}$, 
      target prompts $\{c_{tar}^{(k)}\}_{k=1}^{K}$, 
      editing tokens $\{e^{(k)}\}_{k=1}^{K}$,
      editing schedulers $\{\gamma_k\}_{k=1}^{K}$}
\KwOut{Edited images $\{Z_{t_0}^{(k)}\}_{k=1}^{K}$}

\tcp{Initial forward inversion}
$\{Z_{t_i}^{src}\}_{i=0}^{N} \leftarrow \texttt{ForwardInv}(Z_{t_0}^{src}, c_{src})$

\For{$k = 1$ \KwTo $K$}{
    \tcp{Update base mask for current turn}
    $M_{base} \leftarrow \texttt{GetBaseMask}(Z_{t_0}^{src}, e^{(k)})$ \tcp*{Feed to SAM3}
    \tcp{Backward editing, storing full trajectory for next turn}
    $Z_{t_0}^{(k)}, \{Z_{t_i}^{src}\}_{i=0}^{N} \leftarrow \texttt{BackwardEdit}(\{Z_{t_i}^{src}\}, c_{tar}^{(k)}, M_{base}, \gamma_k)$
}

\Return $\{Z_{t_0}^{(k)}\}_{k=1}^{K}$
\end{algorithm}
\begin{figure}[!ht]
    \centering
    \includegraphics[width=\linewidth]{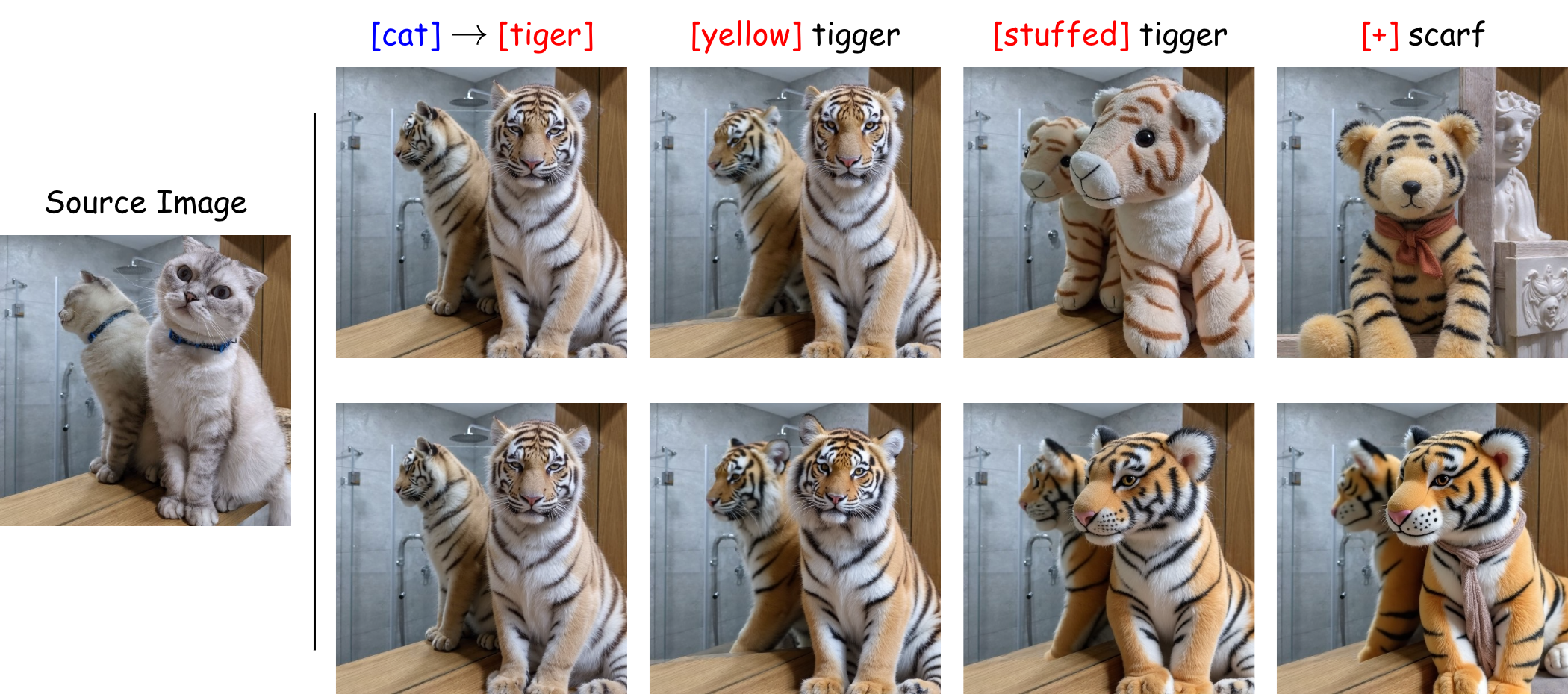}
    \caption{Comparision of SteerFlow multi-turn editing and single-turn editing. First row: single-turn editing. Second row: multi-turn editing.}
    \label{fig:multi_vs_single}
\end{figure}
\noindent\textbf{Comparison to Single-Turn Editing.} To validate that SteerFlow multi-turn editing algorithm (\cref{alg:steerflow_multiturn}) enable consistent edits without accumulating drift, we compare it against single-turn editing. While multi-turn editing uses the previous turn's editing trajectory as the current turn's source trajectory, single-turn editing fix the source trajectory and iteratively update the target prompt. As explained in \cref{sec:trajectory_interpolation}, the editing error bound  is determined by the \textit{editing deviation} (i.e., the difference between  source-conditioned and target-conditioned velocities), so the divergence  between the edited and source images grows as the target prompt becomes increasingly dissimilar from the source prompt.

This is illustrated in the first row of \cref{fig:multi_vs_single}, which corresponds to single-turn editing. The first two turns do not exhibit significant deviation from the source image since the two prompts remain \textit{close} to the source prompt. However, changing the tiger to ``stuffed tiger'' (turn 3) and adding a ``scarf'' (turn 4) result in greater loss of source structure as the gap between the target and source prompts widens. By contrast, multi-turn editing in the second row effectively constrains the drift the turn 3 and turn 4 by updating the source trajectory at each turn, which keeps the source and target velocities better aligned.

\section{Proofs of Theoretical Results}
\label{appendix:proofs}
We first state the assumptions and theorem required for the proofs.
\begin{assumption}[Lipschitz Continuity]
\label{asm:lipschitz}
There exists a constant $L > 0$ such that for all states $z, z' \in \mathbb{R}^d$ and time $t \in [0, 1]$, the vector field $v_\theta(z, t)$ satisfies:
$$||v_\theta(z, t) - v_\theta(z', t)|| \leq L ||z - z'||$$
\end{assumption}

\begin{assumption}[Bounded Trajectory Curvature]
\label{asm:bounded_curvature}
There exists a constant $M$ such that the total time derivative of the velocity field is bounded:
$$\left\| \frac{d}{dt} v_\theta(Z_t, t) \right\| \le M$$
for all $t \in [0, 1]$. This bounds the maximum curvature (acceleration) of the flow.
\end{assumption}

\begin{theorem}[Banach Fixed-Point Theorem]
\label{theorem:banach}
Let $(X, d)$ be a complete metric space and let $T: X \to X$ be a strict contraction, i.e., there exists $q \in [0,1)$ such that for all $x, y \in X$,
$$d(T(x), T(y)) \le q \cdot d(x, y).$$
Then $T$ has a unique fixed point $x^* \in X$, and the sequence $x_{k+1} = T(x_k)$ converges to $x^*$ for any initial $x_0 \in X$.
\end{theorem}

\subsection{Proof of \cref{prop:inversion_error}}
\label{proof_1}
\begin{proof}
We track the accumulation of the inversion error $\mathcal{E}_{t_i} = \| \hat{Z}_{t_i} - Z_{t_i} \|$ as we integrate backward from $t_N=1$ down to $t_0=0$. At the initial point of the backward process ($i=N$), $\mathcal{E}_{t_N} = 0$ since $\hat{Z}_1 = Z_1$. The error arises from the asymmetry between forward and backward latents:
\begin{align}
\text{Forward Latent: } \quad Z_{t_i} &= Z_{t_{i+1}} - \Delta t \cdot v_\theta(Z_{t_i}, t_i) \\
\text{Backward Latent: } \quad \hat{Z}_{t_i} &= \hat{Z}_{t_{i+1}} - \Delta t \cdot v_\theta(\hat{Z}_{t_{i+1}}, t_{i+1})
\end{align}
The inversion error at time $t_i$ is:
\begin{align}
\mathcal{E}_{t_i} &= \|\hat{Z}_{t_i} - Z_{t_i}\| = \| (\hat{Z}_{t_{i+1}} - Z_{t_{i+1}}) - \Delta t \left( v_\theta(\hat{Z}_{t_{i+1}}, t_{i+1}) - v_\theta(Z_{t_i}, t_i) \right) \| \\
&\leq \mathcal{E}_{t_{i+1}} + \Delta t \| v_\theta(\hat{Z}_{t_{i+1}}, t_{i+1}) - v_\theta(Z_{t_i}, t_i) \| \quad \text{(Triangle Inequality)}
\label{eq:step_error}
\end{align}
We bound $\| v_\theta(\hat{Z}_{t_{i+1}}, t_{i+1}) - v_\theta(Z_{t_i}, t_i) \|$ via the cross-term $v_\theta(Z_{t_{i+1}}, t_{i+1})$:
\begin{align}
\| v_\theta(\hat{Z}_{t_{i+1}}, t_{i+1}) - v_\theta(Z_{t_i}, t_i) \| &\le \| v_\theta(\hat{Z}_{t_{i+1}}, t_{i+1}) - v_\theta(Z_{t_{i+1}}, t_{i+1}) \| \notag \\
&\quad + \| v_\theta(Z_{t_{i+1}}, t_{i+1}) - v_\theta(Z_{t_i}, t_i) \|
\label{eq:cross_term}
\end{align}
By the second-order Taylor expansion of $v_\theta(Z_{t_{i+1}}, t_{i+1})$ around $(Z_{t_i}, t_i)$:
\begin{align}
v_\theta(Z_{t_{i+1}}, t_{i+1}) \approx v_\theta(Z_{t_i}, t_i) + \Delta t \frac{\partial v_\theta}{\partial t} + \nabla_z v_\theta \cdot (Z_{t_{i+1}} - Z_{t_i}) + O(\Delta t^2)
\end{align}
Since $Z_{t_{i+1}} - Z_{t_i} = \Delta t \cdot v_\theta(Z_{t_i}, t_i)$, 
\begin{align}
v_\theta(Z_{t_{i+1}}, t_{i+1}) &\approx v_\theta(Z_{t_i}, t_i) + \Delta t \frac{\partial v_\theta}{\partial t} + \nabla_z v_\theta \cdot \left(\Delta t \cdot v_\theta(Z_{t_i}, t_i)\right) + O(\Delta t^2) \\
&\approx v_\theta(Z_{t_i}, t_i) + \Delta t \cdot \frac{dv_\theta}{d t} \quad \text{(Disregard $O(\Delta t^2)$)} \\
&\Rightarrow \| v_\theta(Z_{t_{i+1}}, t_{i+1}) - v_\theta(Z_{t_i}, t_i) \| \leq M \Delta t
\label{eq:taylor_expansion}
\end{align}
Applying \cref{eq:taylor_expansion} and \cref{asm:lipschitz} to \cref{eq:cross_term}, we have:
\begin{align}
\| v_\theta(\hat{Z}_{t_{i+1}}, t_{i+1}) - v_\theta(Z_{t_i}, t_i) \| &\le L \cdot \| \hat{Z}_{t_{i+1}} - Z_{t_{i+1}} \| + M \Delta t \\
&= L \cdot \mathcal{E}_{t_{i+1}} + M \Delta t \quad \text{(\cref{asm:bounded_curvature})}
\label{eq:velocity_bound}
\end{align}
Applying \cref{eq:velocity_bound} back to \cref{eq:step_error}, we obtain the final error bound:
\begin{align}
\mathcal{E}_{t_i} \leq (1 + L\Delta t) \mathcal{E}_{t_{i+1}} + M \Delta t^2
\end{align}
To solve for $\mathcal{E}_{0}$, we unroll this recurrence backward in time from $i=N-1$ to $i=0$. The recurrence strictly follows a geometric series summation:
\begin{align}
\mathcal{E}_{t_0} \le M \Delta t^2 \sum_{j=0}^{N-1} (1 + L \Delta t)^j
\end{align}
Using the standard geometric series formula $\sum_{j=0}^{N-1} r^j = \frac{r^N - 1}{r - 1}$ with the ratio $r = 1 + L \Delta t$, we arrive at the precise global error bound:
\begin{align}
\mathcal{E}^{inv}_{0} \le M \Delta t^2 \frac{(1 + L \Delta t)^N - 1}{(1 + L \Delta t) - 1} = \frac{M \Delta t}{L} \left( (1 + L\Delta t)^N - 1 \right)
\end{align}
Finally, since $\Delta t = 1/N$, by applying the limit definition of the exponential function $(1 + \frac{L}{N})^N \le e^L$ for all $N \ge 1$, we obtain the exponential upper bound that is indepedent of $N$:
\begin{align}
\mathcal{E}^{inv}_{0} \le \frac{M \Delta t}{L} (e^L - 1) \quad \text{(As $N \to \infty$)}
\end{align}
\end{proof}

\begin{proposition}[Convergence of The Update Map]
\label{prop:convergence}
Under the condition that the step size $\Delta t < \frac{1}{L}$ ($L$ is the Lipschitz constant), the update map $\Phi_i$ is a strict contraction mapping. By the Banach Fixed-Point Theorem (\cref{theorem:banach}, the sequence $\{v_i^k\}$ is guaranteed to converge to a unique solution $v^*_i$.
\end{proposition}

\subsection{Proof of \cref{prop:convergence}}
\begin{proof}
The proof applies the Banach Fixed-Point Theorem (\cref{theorem:banach}) directly to the velocity update map by leveraging the Lipschitz assumption (\cref{asm:lipschitz}). The update map $\Phi_i: \mathbb{R}^d \to \mathbb{R}^d$ is defined at a specific timestep $i$ as:
$$\Phi_i(v) = v_\theta(Z_{t_i} + \Delta t \cdot v, t_{i+1})$$
The space $\mathbb{R}^d$ equipped with the standard $L_2$ norm metric $d(u, v) = \|u - v\|$ forms a complete metric space (a Banach space). To apply the Banach Fixed-Point Theorem, we must show that $\Phi_i$ is a strict contraction mapping on this space.

For any $u, v \in \mathbb{R}^d$, by \cref{asm:lipschitz},
\begin{align}
\|\Phi_i(v) - \Phi_i(u)\|
&= \| v_\theta(Z_{t_i} + \Delta t \cdot v, t_{i+1}) - v_\theta(Z_{t_i} + \Delta t \cdot u, t_{i+1}) \| \notag \\
&\le L \| (Z_{t_i} + \Delta t \cdot v) - (Z_{t_i} + \Delta t \cdot u) \| \notag \\
&= \left(L \Delta t\right) \| v - u \|.
\end{align}
Let $q = L \Delta t$. The assumption $\Delta t < \frac{1}{L}$ implies $q \in [0,1)$, so $\Phi_i$ is a strict contraction. Applying Banach Fixed-Point Theorem guarantees that:
\begin{enumerate}
    \item There exists a uniquely defined fixed-point velocity $v^*_i \in \mathbb{R}^d$ satisfying $\Phi_i(v^*_i) = v^*_i$.
    \item The sequence generated by $v_i^{k+1} = \Phi_i(v_i^k)$ monotonically converges to $v^*_i$ for any initialization $v_i^0$.
\end{enumerate} 
By initializing the sequence with the explicit Euler velocity, $v_i^0 = v_\theta(Z_{t_i}, t_i)$, the AFP solver strictly minimizes the local discretization gap $\delta_i = \| v_\theta(Z_{t_i}, t_i) - v_\theta(Z_{t_{i+1}}, t_{i+1}) \|$. This procedure effectively ``pulls'' the velocity at $t_i$ toward the future state at $t_{i+1}$, aligning the velocity with the local curvature of the flow.
\end{proof}

\subsection{Proof of \cref{prop:euler_edit_error}}
\begin{proof}
Starting at $Z_{t_N}^{src}$ from the forward inversion process (\cref{eq:forward}), we can perfectly reconstruct the source image $Z_{t_0}^{src}$ by reversing the ODE:
\begin{equation}
    Z_{t_0}^{src} = Z_{t_N}^{src} - \Delta t \cdot \underbrace{\left(\sum_{i=0}^{N-1} v_\theta(Z_{t_i}^{src}, t_i, c_{src})\ \right)}_{\text{Cached source velocities}}
    \label{eq:reconstruction}
\end{equation}
Starting at $Z_{t_N}^{tar} = Z_{t_N}^{src}$, the editing outcome is as follows:
\begin{equation}
    Z_{t_0}^{tar} = Z_{t_N}^{tar} - \Delta t \cdot \left(\sum_{i=1}^{N} v_\theta(Z_{t_i}^{tar}, t_i, \varnothing) + w\left(v_\theta(Z_{t_i}^{tar}, t_i, c_{tar}) - v_\theta(Z_{t_i}^{tar}, t_i, \varnothing) \right)  \ \right)
    \label{eq:edit_output}
\end{equation}
From \cref{eq:edit_output} and \cref{eq:reconstruction}, by rearranging the terms we have:
\begin{align}
    \|Z_{t_0}^{tar} - Z_{t_0}^{src}\| &= \Delta t \cdot \| \sum_{i=1}^{N} v_\theta(Z_{t_i}^{tar}, t_i, c_{tar}) - v_\theta(Z_{t_{i-1}}^{src}, t_{i-1}, c_{src}) \notag \\  & \quad \quad+ (w-1)  \left(v_\theta(Z_{t_i}^{tar}, t_i, c_{tar}) - v_\theta(Z_{t_i}^{tar}, t_i, \varnothing) \right)  \ \| \\
    &\leq \Delta t \left( \| V_\Delta \| + (w-1) \| V_{CFG} \| \right) \quad \text{(Triangle Inequality)} \\
    \text{where }V_{\Delta} = &\sum_{i=1}^{N} v^{tar}_\theta (Z^{tar}_{t_{i}}) - v^{src}_\theta(Z^{src}_{t_{i-1}}); \; V_{CFG} = \sum_{i=1}^N v^{tar}_\theta (Z^{tar}_{t_i}) - v^{\varnothing}_\theta(Z^{tar}_{t_i}) \notag
\end{align}
\end{proof}

\subsection{Proof of \cref{prop:steerflow_error}}
\begin{proof}
Define the maximum instantaneous velocity deviation along the source trajectory as
\begin{equation}
    \delta_{\max} = \max_t \left\| \tilde{v}_\theta(Z_t^{\mathrm{src}}, c_{\mathrm{tar}}) - v_\theta(Z_t^{\mathrm{src}}, c_{\mathrm{src}}) \right\|
    \label{eq:delta_max}
\end{equation}
Let $V_{t_i}^{\mathrm{src}} = v_\theta(Z^{\mathrm{src}}_{t_{i-1}}, t_{i-1}, c_{\mathrm{src}})$ denote the cached source velocity from the forward process, and let $V_{t_i}^{\mathrm{tar}} = \tilde{v}_\theta(Z^{\mathrm{edit}}_{t_i}, t_i, c_{\mathrm{tar}}, w)$ denote the target-conditioned velocity evaluated at the SteerFlow editing latent $Z^{\mathrm{edit}}_{t_i}$. The SteerFlow editing velocity is then defined as the interpolation
\begin{equation}
    V^{\mathrm{edit}}_t = V^{\mathrm{src}}_t + \alpha \cdot \left(V^{\mathrm{tar}}_t - V^{\mathrm{src}}_t\right),
\end{equation}
from which the source and editing latents at time $t_i$ are given by:
\begin{align}
    Z^{src}_{t_i} &= Z^{src}_{t_{i+1}} - \Delta t \cdot V_{t_i}^{src} \\
    Z^{edit}_{t_i} &= Z^{edit}_{t_{i+1}} - \Delta t \cdot \left(V^{src}_{t_i} + \alpha \cdot (V^{tar}_{t_i} - V^{src}_{t_i}) \right)
\end{align}
Then the editing error at time $t_i$ is
\begin{align}
    \mathcal{E}_{t_i} = \| Z^{edit}_{t_i} - Z^{src}_{t_i} \| \leq \mathcal{E}_{t_{i+1}} +  \alpha \Delta t \cdot \| V_{t_i}^{tar} - V_{t_i}^{src} \| \quad \text{(Triangle Inequality)}
    \label{eq:edit_error}
\end{align}
To bound $\| V_{t_i}^{tar} - V_{t_i}^{src} \|$, we leverage the cross-term $\tilde{v}_\theta(Z^{src}_{t_{i+1}}, c_{tar})$:
\begin{align}
    \| V_{t_i}^{tar} - V_{t_i}^{src} \| &= \| \tilde{v}_\theta(Z^{edit}_{t_{i+1}}, c_{tar}) - \tilde{v}_\theta(Z^{src}_{t_{i+1}}, c_{tar}) \notag \\
    &\quad + \tilde{v}_\theta(Z^{src}_{t_{i+1}}, c_{tar}) - v_\theta(Z^{src}_{t_{i+1}}, c_{src}) \| \\
    &\leq L\mathcal{E}_{t_{i+1}} + \delta_{\max} \quad \text{(\cref{asm:lipschitz} and \cref{eq:delta_max})}
    \label{eq:deviation_error}
\end{align}
Applying \cref{eq:deviation_error} to \cref{eq:edit_error}, we have:
\begin{align}
    \mathcal{E}_{t_i} \le (1 + \alpha L \Delta t) \mathcal{E}_{t_{i+1}} + \alpha \Delta t \delta_{\max}
\end{align}
By solving this recurrence from $t_N$ to $t_0$ similar to \cref{proof_1}, we have the generalized global error bound:
\begin{align}
\mathcal{B}(\alpha) = \frac{\delta_{\max}}{L} \left( e^{\alpha L} - 1 \right)
\end{align}
For standard editing without interpolation ($\alpha = 1$), the error bound is:
\begin{align}
\mathcal{B}^{tar}_0 = \mathcal{B}(1) = \frac{\delta_{\max}}{L} \left( e^{L} - 1 \right)
\end{align}
For SteerFlow with interpolation factor $\alpha \in (0, 1)$, the error bound is:
\begin{align}
\mathcal{B}^{edit}_0 = \mathcal{B}(\alpha) = \frac{\delta_{\max}}{L} \left( e^{\alpha L} - 1 \right)
\end{align}
To compare $\mathcal{B}^{edit}_0$ and $\mathcal{B}^{tar}_0$, we analyze the function $f(x) = e^x - 1$. Because the exponential function is strictly convex ($f''(x) = e^x > 0$) and $f(0) = 0$, it follows from the definition of convexity that for any $\alpha \in (0, 1)$:
\begin{equation}
f(\alpha x + (1-\alpha) \cdot 0) < \alpha f(x) + (1-\alpha) f(0) \implies e^{\alpha L} - 1 < \alpha (e^L - 1)
\end{equation}
Substituting this inequality into our error bounds yields the final result:
\begin{equation}
\mathcal{B}^{edit}_0 = \frac{\delta_{\max}}{L} \left( e^{\alpha L} - 1 \right) < \alpha \frac{\delta_{\max}}{L} \left( e^{L} - 1 \right) = \alpha \cdot \mathcal{B}^{tar}_0
\end{equation}
\end{proof}

\section{Ablation Studies}
\label{appendix:ablation}
We ablate the key components of SteerFlow to validate each design choice. Specifically, we examine: 
\begin{enumerate}
    \item The performance of AFP Solver with increasing $K$, and how it affects the performance of image editing.
    \item Different decay factor $\gamma$ and guidance scale $w$.
    \item Different $\alpha$-schedulers for Trajectory Interpolation.

\end{enumerate}
All experiments are conducted on FLUX.1-dev using task 0 of PIE-Bench with 140 images of various editing styles.

\begin{figure}[t]
    \centering
    \includegraphics[width=\textwidth]{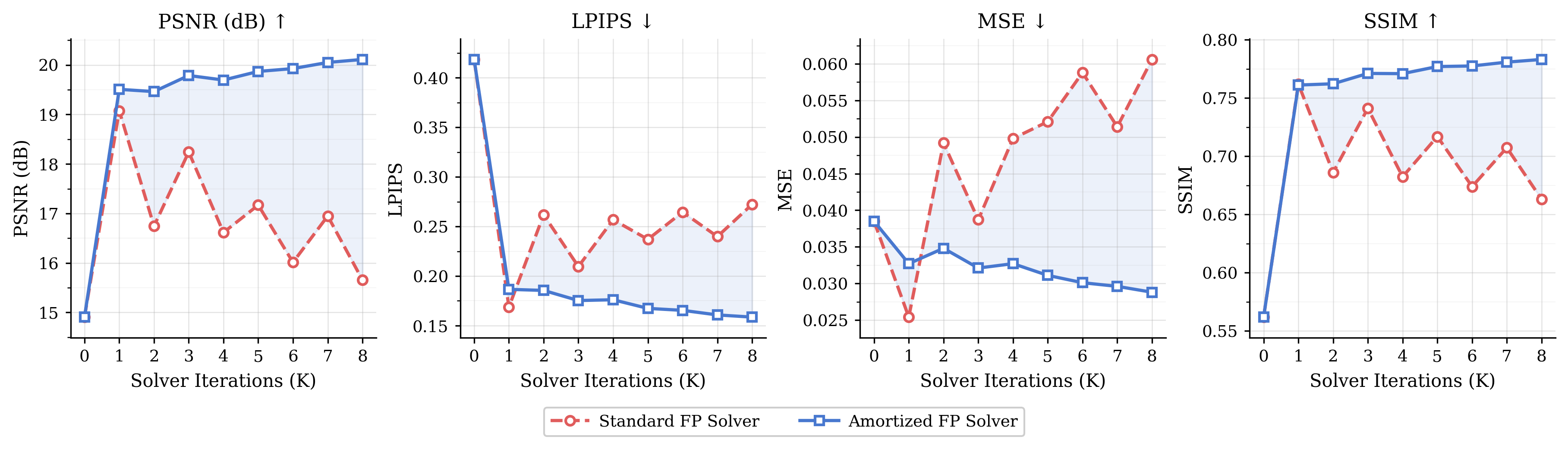}
    \caption{Quantitative comparison of Standard Fixed-Point solver (first row) and AFP solver (second row) with increasing $K$.}
    \label{fig:fixed_point_metrics}
\end{figure}

\begin{figure}[t]
    \centering
    \includegraphics[width=\textwidth]{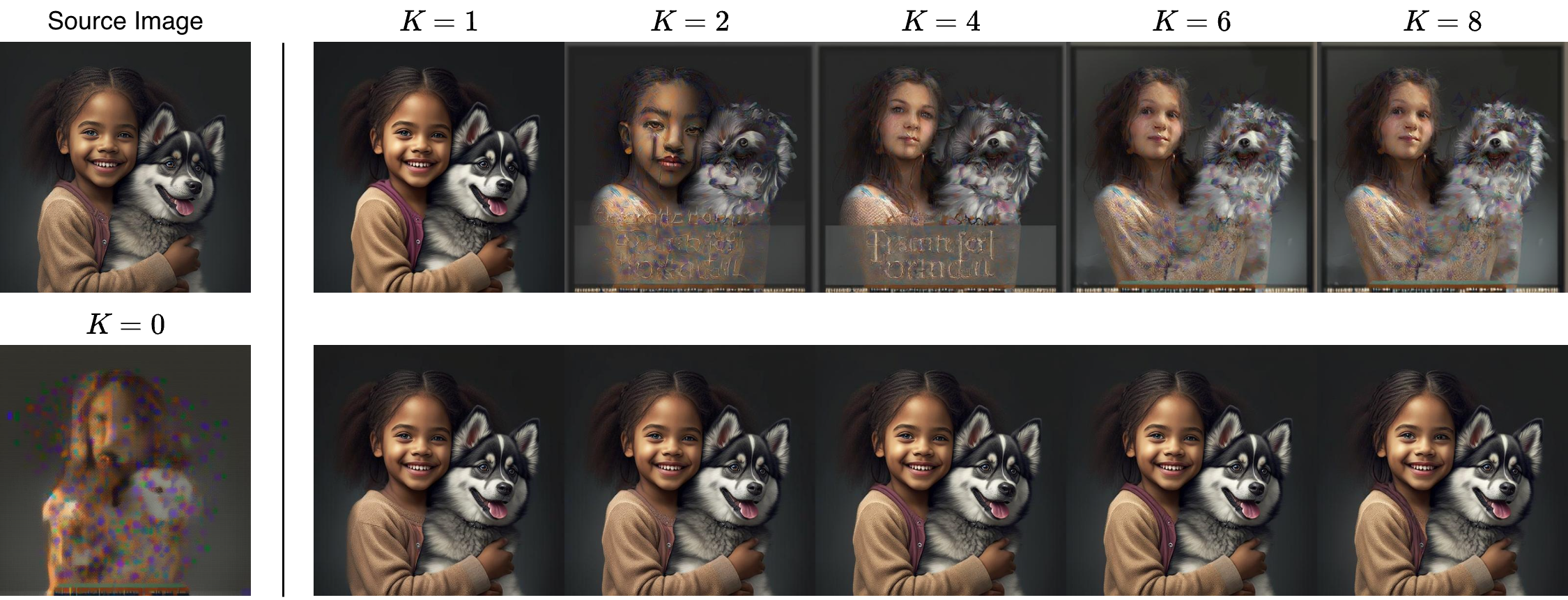}
    \caption{Qualitative comparison of Standard Fixed-Point solver (first row) and AFP solver (second row) with increasing $K$.}
    \label{fig:qualitative_recon}
\end{figure}

\begin{figure}[t]
    \centering
    \includegraphics[width=\textwidth]{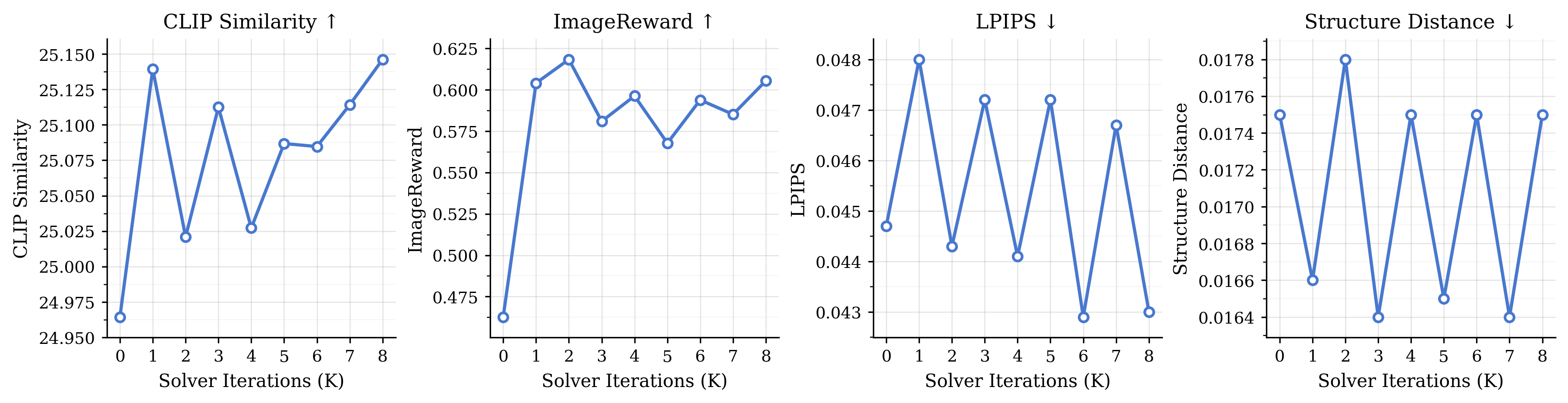}
    \caption{Evaluation of AFP solver on editing performance with increasing $K$.}
    \label{fig:afp_solver_editing}
\end{figure}

\subsection{AFP Solver}
We experimentally validate that the AFP solver mitigates the moving target problem in \cref{fig:fixed_point_metrics}, which compares the reconstruction performance of the AFP solver to the standard fixed-point solver. Across all metrics, the AFP solver achieves increasingly better reconstruction as $K$ grows, consistently outperforming the standard fixed-point solver, which exhibits high variance across $K$. An illustrative comparison between the two solvers is shown in \cref{fig:qualitative_recon}. We can observe that Standard Fixed-Point solver achieve the best reconstruction at $K=1$, and the performance degrades as $K$ increases.

While high $K$ benefits reconstruction performance for AFP Solver, we observe a high variance of editing performance with increasing $K$, as shown in \cref{fig:afp_solver_editing}. The performance gain is most noticeable with $K=1$, which ends up as our final hyperparameter choice for the main experiments.

\begin{figure}[!ht]
    \centering
    \includegraphics[width=\textwidth]{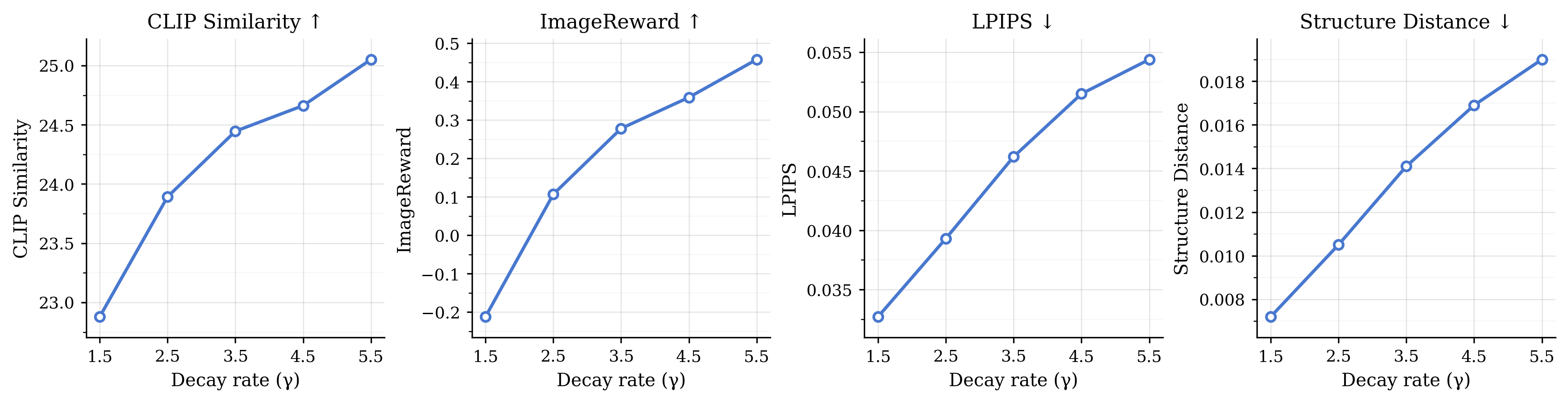}
    \caption{Impact of varying decay rate $\gamma$ on editing performance.}
    \label{fig:sweep_a_metrics}
\end{figure}

\begin{figure}[!ht]
    \centering
    \includegraphics[width=\textwidth]{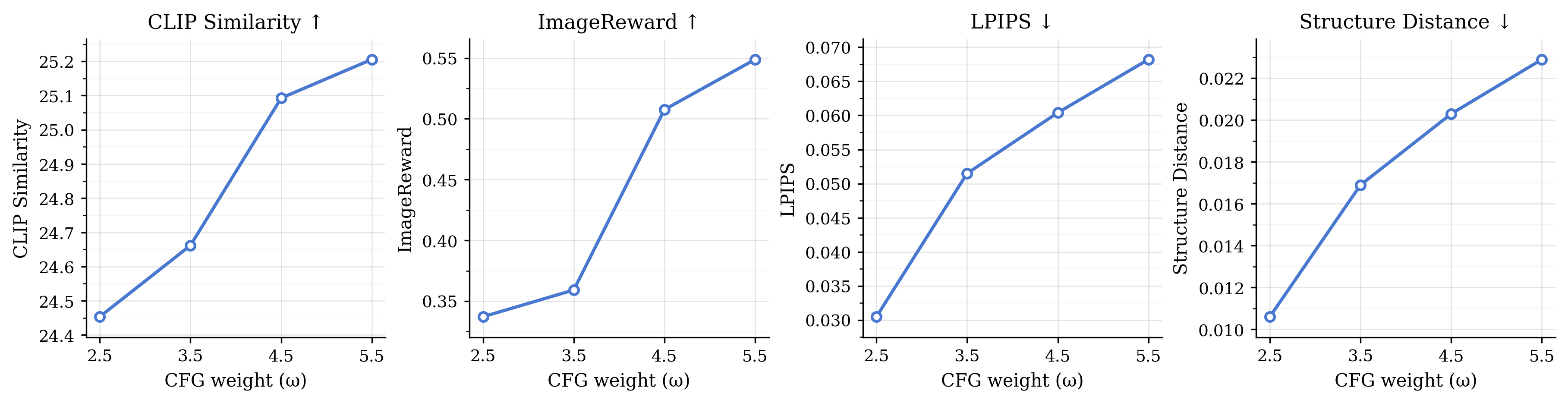}
    \caption{Impact of varying CFG weight $w$ on editing performance.}
    \label{fig:sweep_g_metrics}
\end{figure}

\subsection{Control on Editability ($w$ and $\gamma$)}
As shown in \cref{fig:sweep_a_metrics} and \cref{fig:sweep_g_metrics}, increasing the decay rate $\gamma$ and guidance scale $w$ improves editability (measured by CLIP and ImageReward) but sacrifice source preservation (measured by LPIPS and Structure Distance). Notably, $\gamma$ offers more 
fine-grained control over the trade-off between source preservation and target alignment, whereas setting $w$ too high may introduce unwanted structural changes and artifacts --- a commly known issue of CFG~\cite{saini2025rectified}. The qualitative results of increasing $\gamma$ is shown in \cref{fig:gammas}, where we can observe that setting a high $\gamma$ allows for more structural change.

\begin{figure}[!ht]
    \centering
    \includegraphics[width=\linewidth]{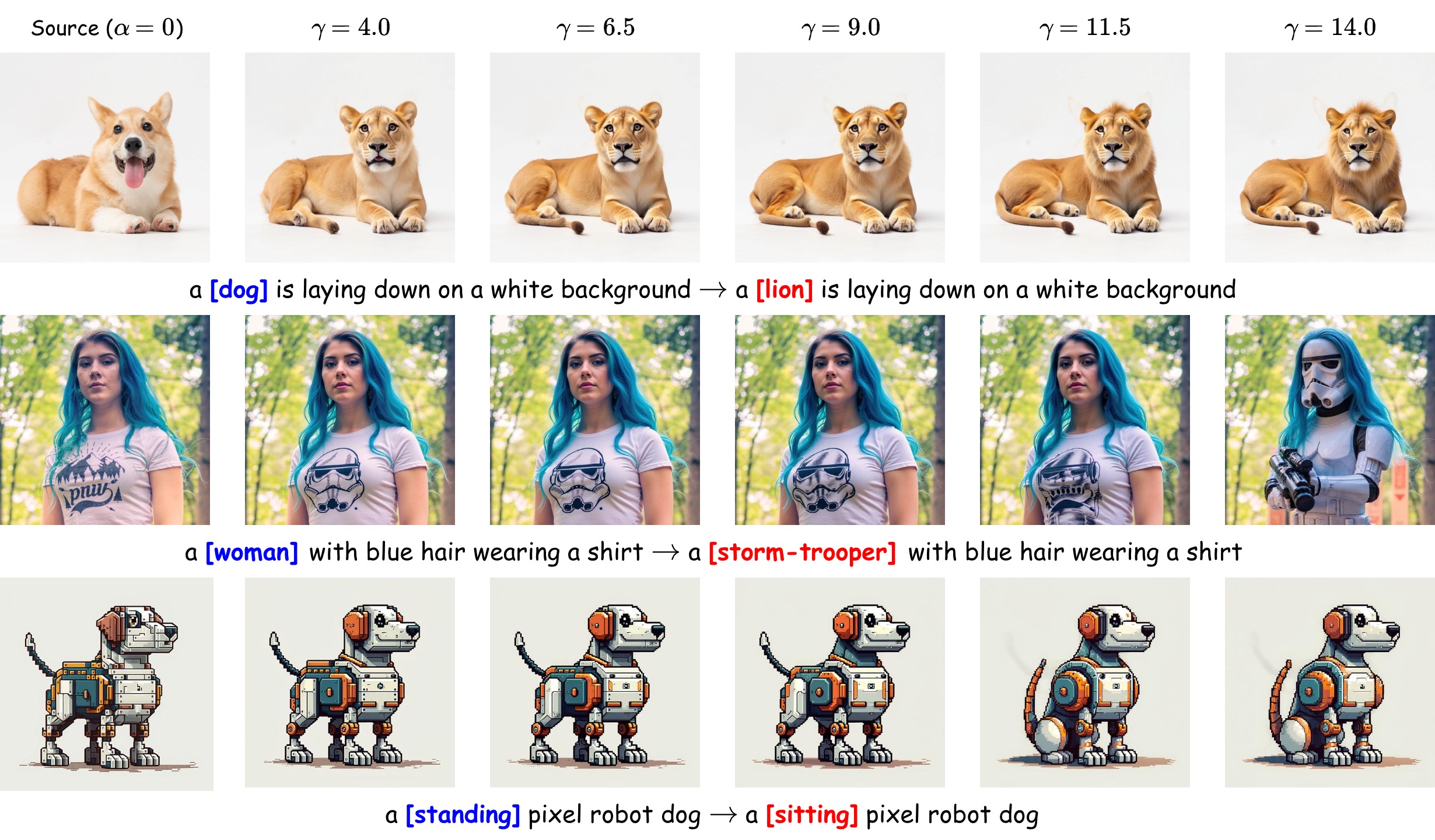}
    \caption{Illustrations of varying decay rate $\gamma$ on editability.}
    \label{fig:gammas}
\end{figure}

\subsection{Different $\alpha$-schedulers}
We evaluate several variants of $\alpha$-schedulers for the backward process in \cref{tab:alpha_schedulers}. Fixing $\alpha$ at low values (0.1 or 0.5) significantly limits the method's editability, whereas a high value (0.9) results in over-editing. While using only CosSim as the editing coefficient provides strong target alignment, it lacks sufficient source preservation. Our final combination of a time-decaying scheduler with cosine similarity achieves the highest source preservation (indicated by the lowest DinoDist and LPIPS scores) while maintaining moderate target alignment that performs similarly to the purely time-decaying approach.

\begin{table}[!ht]
\centering
\caption{SteerFlow editing performance with different $\alpha$-schedulers}
\label{tab:alpha_schedulers}
\begin{tabular}{@{}lcccc@{}}
\toprule
Variant                & DinoDist $\downarrow$ & LPIPS $\downarrow$ & CLIP $\uparrow$ & ImgRew $\uparrow$ \\ \midrule
Fixed $\alpha=0.1$     & 0.05                  & 2.98               & 21.49           & -0.64             \\
Fixed $\alpha=0.5$     & 0.85                  & 10.90              & 23.91           & 0.12              \\
Fixed $\alpha=0.9$     & 7.76                  & 40.31              & 26.57           & 1.08              \\
Only time-decaying     & 3.10                  & 24.55              & 25.32           & 0.70              \\
Only Cossim            & 4.99                  & 30.97              & 26.50           & 1.08              \\
time-decaying + Cossim & 2.20                  & 19.87              & 25.12           & 0.60              \\ \bottomrule
\end{tabular}%
\end{table}

\section{Implementation and Reproducibility Details}
\label{appendix:implementation_details}
\subsection{Experiment Settings}  

For all baseline methods, the hyperparameters strictly adhere to the configurations reported in their original publications and the default settings provided in their respective official repositories. The specific configurations for inversion-based and inversion-free baselines are summarized in \cref{tab:inversion_based_settings} and \cref{tab:inversion_free_settings}, respectively. Notably, as indicated by (*), FlowAlign~\cite{kim2025flowalign} and UniEdit~\cite{jiao2026unieditflow}  utilize a modified Classifier-Free Guidance (CFG) formulation for the target velocity. Unlike the standard approach that extrapolates from an unconditional (null-prompt) velocity, these methods compute the guided target velocity by extrapolating between the source-conditioned and target-conditioned velocities:
\begin{equation}
\label{eq:modified_cfg}
V_t^{tar} = v_\theta(Z_t, t, c_{src}) + w \left( v_\theta(Z_t, t, c_{tar}) - v_\theta(Z_t, t, c_{src}) \right)
\end{equation}

Specific experimental configurations for SteerFlow are shown in \cref{tab:steerflow_settings}, and the hyperparameters of the adaptive masking mechanism are shown in \cref{tab:mask_settings}.
\begin{table}[t]
\centering
\caption{Experiment settings of inversion-based baselines}
\label{tab:inversion_based_settings}
\resizebox{\textwidth}{!}{%
\begin{tabular}{|l|c|ccccccc|}
\hline
\multirow{2}{*}{Method} &
  \multirow{2}{*}{Model} &
  \multicolumn{7}{c|}{Hyperparameters} \\ \cline{3-9} 
 &
   &
  \multicolumn{1}{c|}{timesteps $N$} &
  \multicolumn{1}{c|}{target CFG $w$} &
  \multicolumn{1}{c|}{start step} &
  \multicolumn{1}{c|}{injection steps} &
  \multicolumn{1}{c|}{stop step} &
  \multicolumn{1}{c|}{forward guidance $\gamma$} &
  backward guidance $\eta$ \\ \hline
\multirow{2}{*}{ODE-Inversion} &
  SD3.5 &
  \multicolumn{1}{c|}{50} &
  \multicolumn{1}{c|}{5.5} &
  \multicolumn{1}{c|}{0/50} &
  \multicolumn{1}{c|}{-} &
  \multicolumn{1}{c|}{-} &
  \multicolumn{1}{c|}{-} &
  - \\
 &
  FLUX &
  \multicolumn{1}{c|}{28} &
  \multicolumn{1}{c|}{3.5} &
  \multicolumn{1}{c|}{0/28} &
  \multicolumn{1}{c|}{-} &
  \multicolumn{1}{c|}{-} &
  \multicolumn{1}{c|}{-} &
  - \\ \hline
RF-Inversion &
  FLUX &
  \multicolumn{1}{c|}{28} &
  \multicolumn{1}{c|}{3.5} &
  \multicolumn{1}{c|}{0/28} &
  \multicolumn{1}{c|}{-} &
  \multicolumn{1}{c|}{7/28} &
  \multicolumn{1}{c|}{0.5} &
  0.9 \\ \hline
RF-Solver &
  \multicolumn{1}{l|}{FLUX} &
  \multicolumn{1}{c|}{15} &
  \multicolumn{1}{c|}{2.0} &
  \multicolumn{1}{c|}{0/15} &
  \multicolumn{1}{c|}{2} &
  \multicolumn{1}{c|}{-} &
  \multicolumn{1}{c|}{-} &
  - \\ \hline
FireFlow &
  \multicolumn{1}{l|}{FLUX} &
  \multicolumn{1}{c|}{15} &
  \multicolumn{1}{c|}{2.0} &
  \multicolumn{1}{c|}{0/15} &
  \multicolumn{1}{c|}{2} &
  \multicolumn{1}{c|}{-} &
  \multicolumn{1}{c|}{-} &
  - \\ \hline
\multirow{2}{*}{UniEdit} &
  SD3.5 &
  \multicolumn{1}{c|}{15} &
  \multicolumn{1}{c|}{5.0*} &
  \multicolumn{1}{c|}{6/15} &
  \multicolumn{1}{c|}{-} &
  \multicolumn{1}{c|}{-} &
  \multicolumn{1}{c|}{-} &
  - \\
 &
  FLUX &
  \multicolumn{1}{c|}{30} &
  \multicolumn{1}{c|}{5.0*} &
  \multicolumn{1}{c|}{12/30} &
  \multicolumn{1}{c|}{-} &
  \multicolumn{1}{c|}{-} &
  \multicolumn{1}{c|}{-} &
  - \\ \hline
\end{tabular}%
}
\end{table}
\begin{table}[t]
\centering
\caption{Experiment settings of inversion-free baselines}
\label{tab:inversion_free_settings}
\resizebox{\textwidth}{!}{%
\begin{tabular}{|l|c|ccccc|}
\hline
\multirow{2}{*}{Method}   & \multirow{2}{*}{Model} & \multicolumn{5}{c|}{Hyperparameters}                                                                                \\ \cline{3-7} 
 &
   &
  \multicolumn{1}{c|}{timesteps $N$} &
  \multicolumn{1}{c|}{source CFG $w_s$} &
  \multicolumn{1}{c|}{target CFG $w_t$} &
  \multicolumn{1}{c|}{start step} &
  correction weight $\zeta$ \\ \hline
\multirow{2}{*}{FlowEdit} & SD3.5                  & \multicolumn{1}{c|}{50} & \multicolumn{1}{c|}{1.5} & \multicolumn{1}{c|}{5.5}   & \multicolumn{1}{c|}{17/50} & -    \\
                          & FLUX                   & \multicolumn{1}{c|}{28} & \multicolumn{1}{c|}{5.5} & \multicolumn{1}{c|}{13.5}  & \multicolumn{1}{c|}{4/28}  & -    \\ \hline
FlowAlign                 & SD3.5                  & \multicolumn{1}{c|}{33} & \multicolumn{1}{c|}{-}   & \multicolumn{1}{c|}{13.5*} & \multicolumn{1}{c|}{0/50}  & 0.01 \\ \hline
\end{tabular}%
}
\end{table}
\begin{table}[t]
\centering
\caption{Experiment settings of SteerFlow}
\label{tab:steerflow_settings}
\begin{tabular}{|l|c|cccc|}
\hline
\multirow{2}{*}{Method} & \multirow{2}{*}{Model} & \multicolumn{4}{c|}{Hyperparameters} \\ \cline{3-6} 
 & & \multicolumn{1}{c|}{timesteps $N$} & \multicolumn{1}{c|}{decay rate $\gamma$} & \multicolumn{1}{c|}{CFG $w$} & AFP iterations $K$ \\ \hline
\multirow{2}{*}{SteerFlow} & SD3.5 & \multicolumn{1}{c|}{30} & \multicolumn{1}{c|}{5.5} & \multicolumn{1}{c|}{3.5} & 1 \\
 & FLUX & \multicolumn{1}{c|}{15} & \multicolumn{1}{c|}{4.5} & \multicolumn{1}{c|}{6.5} & 1 \\ \hline
\end{tabular}
\end{table}

\begin{table}[t]
\centering
\caption{SteerFlow Adaptive Mask setting}
\label{tab:mask_settings}
\begin{tabular}{|c|l|c|}
\hline
quantile ratio $q$ & sigmoid temperature $\tau$ & kernel size $k$ \\ \hline
0.95               & \multicolumn{1}{c|}{15}    & 5               \\ \hline
\end{tabular}%
\end{table}

\subsection{Hardware} 
We conduct our experiments on two servers. All experiments on FLUX.1-dev model are conducted on a single A6000 GPU, while experiments on Stable Diffusion 3.5 Medium are conducted on a single RTX3090 GPU.

\subsection{Reproducibility} 
Despite following the original experimental settings and using official implementations, our reproduced results exhibit minor discrepancies from those reported in prior works. We attribute these differences to potential updates in the Diffusers, Transformers, or PyTorch libraries, as well as hardware variation across environments. All results reported in this paper reflect our own reproduced runs.

\section{Limitations}
\label{appendix:limitations}
While SteerFlow presents an architecture-agnostic framework that excels in high-fidelity editing, it has two core limitations. First, the Adaptive Masking mechanism may sometimes restrict large structural changes required by pose and background edits. This can be mitigated by allowing greater mask expansion via the quantile ratio $q$ and morphological closing kernel size $k$, or by disabling masking entirely for such edits. Second, the optimal hyperparameters of SteerFlow may vary across model architectures, requiring some tuning of the CFG weight $w$ and decay rate $\gamma$ for each target model.

\section{A Study on Perfect Latent for Inversion-Based Editing}
\label{appendix:perfect_latent_study}
To examine whether inversion error is the primary cause of structural drift in inversion-based editing, we design a controlled experiment that eliminates it entirely: we generate the source image from a fixed noise sample $\epsilon \sim \mathcal{N}(0, I)$, then reuse that exact noise as the starting latent, changing only the editing token from 
\texttt{[cat]} to \texttt{[dog]}. As shown in \cref{fig:perfect_latent}, even with this perfect latent, the edited image still fails to preserve the source structure and background (\ie, the house). This demonstrates that structural drift stems from the divergence between source and target velocity fields, not inversion error alone. SteerFlow addresses this root cause directly via Trajectory Interpolation, faithfully preserving both the cat's structure and the background of the house.

\begin{figure}[!ht]
    \centering
    \includegraphics[width=\linewidth]{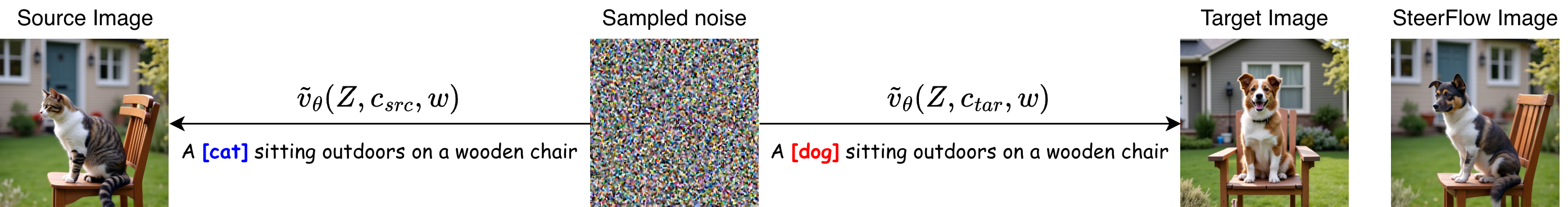}
    \caption{A sample result of inversion-based editing with a perfect latent that is used for the source image.}
    \label{fig:perfect_latent}
\end{figure}

\end{document}